\documentclass[sigconf, nonacm]{acmart}

\newcommand\vldbdoi{10.14778/3489496.3489511}
\newcommand\vldbpages{312 - 320}
\newcommand\vldbvolume{15}
\newcommand\vldbissue{2}
\newcommand\vldbyear{2022}
\newcommand\vldbauthors{Xupeng Miao, Hailin Zhang, Yining Shi, Xiaonan Nie, Zhi Yang, Yangyu Tao, Bin Cui}
\newcommand\vldbtitle{\shorttitle} 
\newcommand\vldbavailabilityurl{https://github.com/PKU-DAIR/Hetu/}
\newcommand\vldbpagestyle{empty} 

\usepackage{bm}
\usepackage[ruled]{algorithm2e}
\usepackage{multirow}
\usepackage{tikz}
\usepackage{url}
\newcommand*{\circled}[1]{\lower.7ex\hbox{\tikz\draw (0pt, 0pt)%
    circle (.5em) node {\makebox[1em][c]{\small #1}};}}

\usepackage{enumerate}

\usepackage{subfigure}
\usepackage{lipsum}

\usepackage{xcolor}

\usepackage[utf8]{inputenc} 
\usepackage[T1]{fontenc}    
\usepackage{booktabs}       
\usepackage{amsfonts}       
\usepackage{nicefrac}       
\usepackage{microtype}      
\usepackage[utf8]{inputenc} 
\usepackage[T1]{fontenc}    
\usepackage{graphicx}
\usepackage{balance}  
\usepackage{soul}
\usepackage{bbm}
\usepackage[utf8]{inputenc}
\usepackage{amsmath}

\usepackage{amssymb}
\usepackage{amsthm}
\usepackage{booktabs}
\usepackage{amssymb}
\usepackage{makecell}

\usepackage{wrapfig}
\usepackage{hyperref}
\usepackage{enumitem}
\usepackage{float}
\usepackage{footnote}
\usepackage[textsize=small]{todonotes}

\newcommand\blfootnote[1]{%
  \begingroup
  \renewcommand\thefootnote{}\footnote{#1}%
  \addtocounter{footnote}{-1}%
  \endgroup
}

\newtheorem{lemma}{Lemma}

\newtheorem{theorem}{Theorem}
\newtheorem{definition}{Definition}

\usepackage{xspace}
\newcommand{\name}{HET\xspace}
\newcommand{\sname}{Het\xspace}

\usepackage{amsmath}

\AtBeginDocument{%
  \providecommand\BibTeX{{%
    \normalfont B\kern-0.5em{\scshape i\kern-0.25em b}\kern-0.8em\TeX}}}

\setcopyright{acmcopyright}
\copyrightyear{2020}
\acmYear{2020}
\acmDOI{10.1145/1122445.1122456}

\acmConference[Woodstock '18]{Woodstock '18: ACM Symposium on Neural
  Gaze Detection}{June 03--05, 2018}{Woodstock, NY}
\acmBooktitle{Woodstock '18: ACM Symposium on Neural Gaze Detection,
  June 03--05, 2018, Woodstock, NY}
\acmPrice{15.00}
\acmISBN{978-1-4503-XXXX-X/18/06}

\begin{document}

\title{\name: Scaling out Huge Embedding Model Training via Cache-enabled Distributed Framework
}


\author{Xupeng Miao$^{*1}$, Hailin Zhang$^{*1}$, Yining Shi$^{1}$, Xiaonan Nie$^{1}$, Zhi Yang$^{1}$, Yangyu Tao$^{2}$, Bin Cui$^{1,3}$}

\affiliation{
\institution{
{$^{1}$}Department of Computer Science \& Key Lab of High Confidence Software Technologies (MOE), Peking University\\
{$^{3}$}Institute of Computational Social Science, Peking University (Qingdao),
{$^{2}$}Tencent Inc. \\
{$^{3}$}Center for Data Science, Peking University \& National Engineering Laboratory for Big Data Analysis and Applications\\
\{xupeng.miao, z.hl, shiyining, xiaonan.nie, yangzhi, bin.cui\}@pku.edu.cn, {$^{2}$}brucetao@tencent.com}
}

\begin{abstract}
Embedding models have been an effective learning paradigm for high-dimensional data. However, one open issue of embedding models is that their representations (latent factors) often result in large parameter space. We observe that existing distributed training frameworks face a scalability issue of embedding models since updating and retrieving the shared embedding parameters from servers usually dominates the training cycle. In this paper, we propose \name, a new system framework that significantly improves the scalability of huge embedding model training. We embrace skewed popularity distributions of embeddings as a performance opportunity and leverage it to address the communication bottleneck with an \emph{embedding cache}. To ensure consistency across the caches, we incorporate a new consistency model into \name design, which provides fine-grained consistency guarantees on a per-embedding basis. Compared to previous work that only allows staleness for read operations, \name also utilizes staleness for write operations. Evaluations on six representative tasks show that \name achieves up to 88\% embedding communication reductions and up to $20.68\times$ performance speedup over the
state-of-the-art baselines.\blfootnote{$^*$~Equal contribution.}
 
\end{abstract}

\settopmatter{printfolios=true}
\maketitle

\pagestyle{\vldbpagestyle}
\begingroup\small\noindent\raggedright\textbf{PVLDB Reference Format:}\\
\vldbauthors. \vldbtitle. PVLDB, \vldbvolume(\vldbissue): \vldbpages, \vldbyear.\\
\href{https://doi.org/\vldbdoi}{doi:\vldbdoi}
\endgroup
\begingroup
\renewcommand\thefootnote{}\footnote{\noindent
This work is licensed under the Creative Commons BY-NC-ND 4.0 International License. Visit \url{https://creativecommons.org/licenses/by-nc-nd/4.0/} to view a copy of this license. For any use beyond those covered by this license, obtain permission by emailing \href{mailto:info@vldb.org}{info@vldb.org}. Copyright is held by the owner/author(s). Publication rights licensed to the VLDB Endowment. \\
\raggedright Proceedings of the VLDB Endowment, Vol. \vldbvolume, No. \vldbissue\ %
ISSN 2150-8097. \\
\href{https://doi.org/\vldbdoi}{doi:\vldbdoi} \\
}\addtocounter{footnote}{-1}\endgroup

\ifdefempty{\vldbavailabilityurl}{}{
\vspace{.3cm}
\begingroup\small\noindent\raggedright\textbf{PVLDB Artifact Availability:}\\
The source code of this research paper has been made publicly available at \url{\vldbavailabilityurl}.
\endgroup
}

\section{Introduction}

To train a model on high-dimensional data, such as words
in a corpus of text~\cite{DBLP:conf/kdd/PerozziAS14,DBLP:journals/corr/abs-1301-3781,DBLP:conf/wsdm/AnandKSZZ19} or the
user-item interaction data~\cite{DBLP:conf/sigir/Wang0WFC19,DBLP:conf/kdd/ZhouZSFZMYJLG18,DBLP:journals/dase/GharibshahZHC20}, it is common to use an embedding model, which projects a sparse high-dimensional feature space, into a
continuous low-dimensional \emph{embedding} space.
For example, in a language model, a training example might be a sparse vector
with non-zero entries corresponding to the IDs of words
in a vocabulary, and the distributed representation for each
word will be a lower-dimensional vector. ``Wide and
deep learning''~\cite{DBLP:conf/recsys/Cheng0HSCAACCIA16} creates distributed representations from
cross-product transformations on categorical features. Embedding model is common at modern web companies (e.g., Facebook~\cite{DBLP:journals/corr/abs-1906-00091}, Google~\cite{DBLP:conf/recsys/CovingtonAS16} and Tencent~\cite{DBLP:conf/kdd/ZhangBLLBW20}), which have been recognized as an
effective learning paradigm to extract useful
information for downstream tasks such as recommendation.

As each feature needs to be represented by a
set of embeddings (i.e., latent vectors), many embedding models are at a \emph{giant} scale and are too large to copy to a worker on every
use, or even to store in RAM on a single host. For instance, the parameters of a real-world document embedding model in Google~\cite{DBLP:journals/corr/DaiOL15,DBLP:conf/osdi/AbadiBCCDDDGIIK16} occupies several terabytes, and the industrial click-through rate prediction model in Baidu~\cite{DBLP:conf/mlsys/ZhaoXJQDS020} has $10^{11}$ input sparse features and also requires 10 Tb parameters.
For this reason, it is challenging to scale embedding models up to large-scale use cases, in which millions or even billions of parameters need to be learned.

Modern distributed ML systems (e.g., TensorFlow~\cite{DBLP:conf/osdi/AbadiBCCDDDGIIK16}) typically adopt the parameter server~\cite{DBLP:conf/osdi/LiAPSAJLSS14} framework to scale out
models. The server usually maintains
the globally shared parameters by aggregating updates
from the workers and updating the global parameters.
Workers communicate only
with the server nodes, updating and retrieving the shared
parameters. Existing ML systems usually support data parallelism where
a worker
usually contains a replica of the ML model and is assigned an
equal-sized partition of the entire training data. Bulk Synchronous Parallel (BSP)~\cite{DBLP:journals/mp/GhadimiLZ16} or Asynchronous Parallel
(ASP)~\cite{DBLP:conf/icml/LianZZL18} are usually adopted for updating the model parameters during distributed training.

However, this setup faces a
scalability issue for large embedding models~\cite{DBLP:conf/sc/XieRLYXWLAXS20,DBLP:conf/mlsys/ZhaoXJQDS020}. We observe that the greatest inefficiency comes from
\textit{updating and retrieving the shared
feature embedding parameters}
through a limited bandwidth link. For example, using TensorFlow with ASP, up to 86\% of training time is spent
on embedding fetching and updating, which dominates
the training cycle.
The major reason is that an embedding model often uses deep neural
networks with low computational complexity, comparing with the giant embedding data. Accordingly,
the computation takes a much shorter time than the reads and writes of remote embedding data. Moreover, with the increasing gap between emerging powerful accelerators and the slow growth of network bandwidth, the embedding communication bottleneck would become even more severe.
To our knowledge, there is little prior work
addressing the scalability issue of embedding models in a distributed environment.

In this paper, we propose \textbf{\ul{\name}}, a novel distributed system framework to
scale \textbf{\ul{H}}uge \textbf{\ul{E}}mbedding model \textbf{\ul{T}}raining. Our key idea is to exploit an efficient \textbf{embedding-cache-enabled} architecture, which is mainly inspired by the critical characteristics for embedding models: \emph{popularity skewness} and \emph{staleness tolerant}. 
Specifically, the popularity distribution of embeddings is often highly skewed,
typically following power-law distributions~\cite{DBLP:conf/sigir/ZhangC15}, implying a performance
opportunity: \emph{a small cache of hot embedding at
each worker can effectively save the network bandwidth while scaling training throughput with the number of workers}. 

Replicating shared embedding data in multiple caches raises the problem of consistency in the presence of writes. 
Fortunately, embedding models are iterative convergent algorithms 
where some staleness errors  
during training are acceptable and will not prevent convergence. 
In other words, embedding models are robust to a bounded amount of inconsistency (e.g., reading out-of-date shared state). 
By relaxing the consistency guarantees properly,
we can exploit the opportunity of caches to gain significant system improvements.
However, conventional relaxed consistency models such as Stale Synchronous Parallel (SSP)~\cite{DBLP:conf/nips/HoCCLKGGGX13} are not aware of the presence of skew access and
require that every single
worker should be able to hold an entire set of parameters.
Moreover, they mainly target straggler problems rather than communication overhead. For example, SSP maintains an up-to-date global model through sent out write updates to servers each clock. Considering the large scale and communication cost of embedding models, the above issues become a critical limitation for scaling out the training.

To address the above issues, we incorporate a new consistency model into \name. Our consistency model
differs from the traditional ones in two aspects.
First, we enable the \textbf{fine-grained} caching and consistency that provides guarantees on a per-embedding basis. Specifically, for each cached embedding, we leverage an embedding-specific Lamport clock to manage its fine-grained
consistency actions (e.g., validation, synchronization) and provide the concept of ``per-embedding-clock-bounded'' consistency --- a worker can see all updates of an embedding older than certain embedding-specific clocks. We provide a fine-grained consistency model and theoretically prove its convergence guarantees.
Second, compared to previous works that only allow staleness for read operations, we further utilizes staleness for writes, allowing \textbf{stale-writes} based on the timestamp deviation between the global and local clocks of each embedding. This means that writing to an embedding residing in the cache does not update the underlying global model until the embedding is invalided or evicted from the cache. This feature is critical in reducing communication overheads. 

\begin{figure}[t]
\centering
\includegraphics[width=0.95\linewidth]{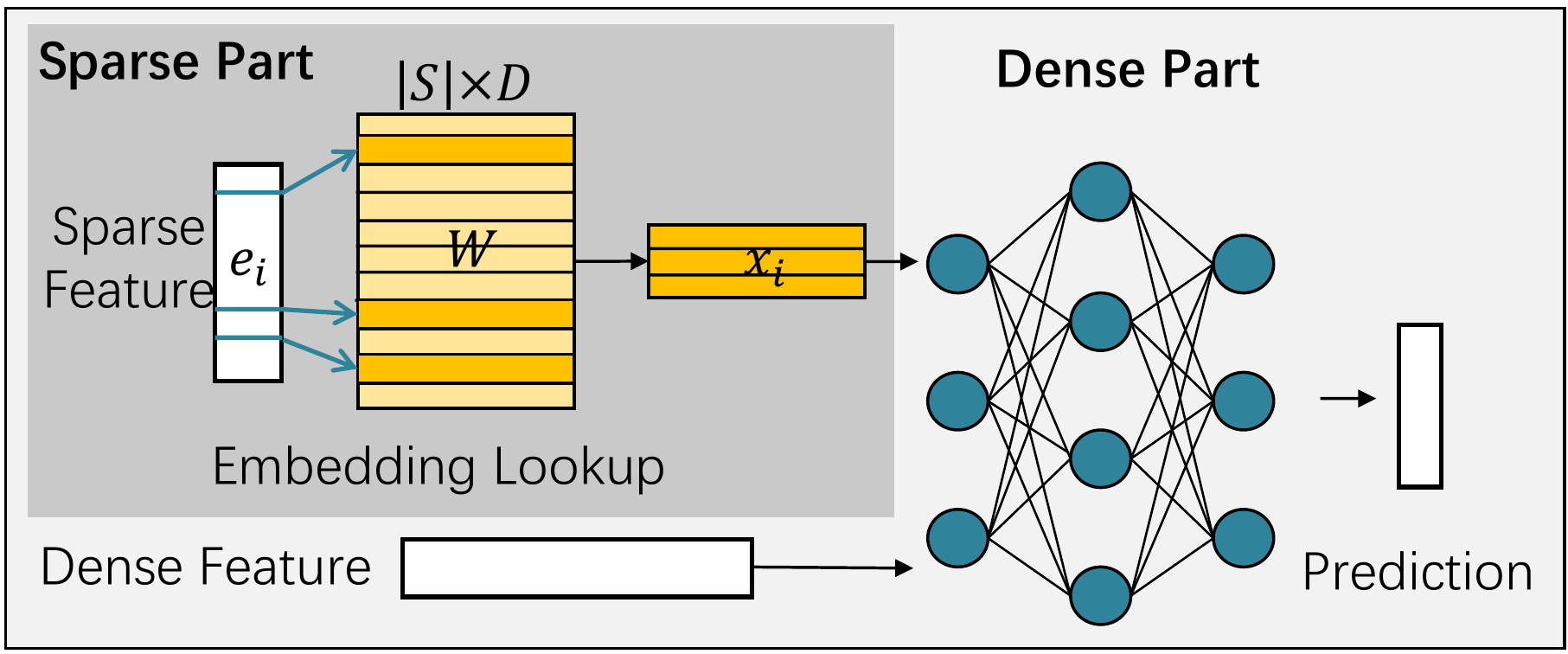}
\vspace{-2mm}
\caption{Illustrate of embedding model architecture}
\vspace{-5mm}
\label{fig:embmodel}
\end{figure}

We summarize our contributions as follows:
First, we reveal the performance bottleneck and the opportunity for scaling huge embedding models, and introduce a novel \textbf{system abstraction} with embedding cache.
Second, we employ a new cache \textbf{consistency model} that provides (1) clock-bounded consistency at the fine-grained of each embedding, and (2) allows staleness for
both caches read and write operations for minimizing communication overhead.
Finally, we build \textbf{\name system}, a new framework that implements the proposed system abstraction and the consistency model, and supports $10^{12}$ parameters scale embedding model training, achieving $6.37-20.68\times$ speedup and up to 88\% embedding communication reduction over the state-of-the-art baseline systems.

\section{Preliminary}
\subsection{Distributed Training}
\paragraph{\textbf{Distributed machine learning.}}The target of machine learning is to find a model $\mathbf{x}\in\mathbb{R}^d$ ($d$ is the total number of parameters in the model) that minimizes the empirical risk:
\begin{equation}
\small{
    \min_{\mathbf{x}} \Big[F(\mathbf{x}) := \frac{1}{|\xi|}\sum_{i}f(\mathbf{x};\xi_i)\Big],
}
\end{equation}
where $f(\cdot)$ is the loss function, $\xi$ is the training dataset and $\xi_i$ represents the $i$-th data sample. 
Distributed ML systems have been extensively studied in recent years to scale up ML for big data and large models.
Parameter Server (PS) is a trendy data parallelism architecture for many existing systems (e.g., TensorFlow~\cite{DBLP:conf/osdi/AbadiBCCDDDGIIK16}, PS2~\cite{DBLP:conf/sigmod/Zhang0SYJM19}). 
Another choice is All-Reduce and several recent systems (e.g., PyTorch~\cite{DBLP:journals/pvldb/LiZVSNLPSVDC20}, Horovod~\cite{DBLP:journals/corr/abs-1802-05799}) show superior performance over PS with the help of NCCL~\cite{nccl}, especially for dense models.

\paragraph{\textbf{Parallel training paradigms.}}
Most of data parallelism studies manage to keep consistent model performance as the standalone mini-batch SGD. 
BSP assumes that all $N$ workers are fully synchronized and performing the following update rule:
\begin{equation}
\small{
    \mathbf{x}(t+1) = \mathbf{x}(t) - \eta \Big[ \frac{1}{N}\sum_{i=1}^{N}G^i(\mathbf{x}(t);\xi^{i})\Big],
}
\end{equation}
where $\eta$ is the learning rate, $\xi^i$ are randomly sampled from the training set, and $G^i(\cdot)$ denotes the gradient from the $i$-th worker. 
However, the frequent synchronization and straggler problem bring significant communication costs. ASP avoids such overheads by allowing the workers to proceed without waiting for each other. But the model degradation happens because of the stale gradients. To balance the trade-off between training efficiency and model performance, SSP and several variants~\cite{DBLP:conf/sigmod/JiangCZY17} have been proposed. It has been proved that they could share the same convergence rate with BSP when the staleness is upper bounded~\cite{DBLP:conf/nips/LianHLL15}. Unfortunately, SSP requires to store the replication of the entire model inside every single worker, which is impractical for giant models.

\subsection{Embedding Models}
Many types of embedding models (e.g., Wide \& Deep~\cite{DBLP:conf/recsys/Cheng0HSCAACCIA16}, Deep \& Cross~\cite{DBLP:conf/kdd/WangFFW17},
DeepFM~\cite{DBLP:conf/ijcai/GuoTYLH17}, xDeepFM~\cite{DBLP:conf/kdd/LianZZCXS18} and Deep Interest Network~\cite{DBLP:conf/kdd/ZhouZSFZMYJLG18}) have been developed for high-dimension data, and have achieved widespread success in recommender
systems--a critical service for internet companies.
Figure~\ref{fig:embmodel} illustrates a common embedding model architecture.
In order to handle categorical data, 
embedding tables map categorical features to dense representations in
an abstract space.
In particular, each embedding
lookup may be interpreted as using a one-hot
vector $e_i$ (with the $i$-th position being 1 while
others are 0, where index $i$ corresponds to $i$-th
category) to obtain the corresponding row vector
of the embedding table
$\mathbf{W}\in\mathbb{R}^{|S|\times D}$
as: $\mathbf{x}_i=\mathbf{W}^\top \mathbf{e}_i.$

\begin{figure}[t]
\centering
\includegraphics[width=0.8\linewidth]{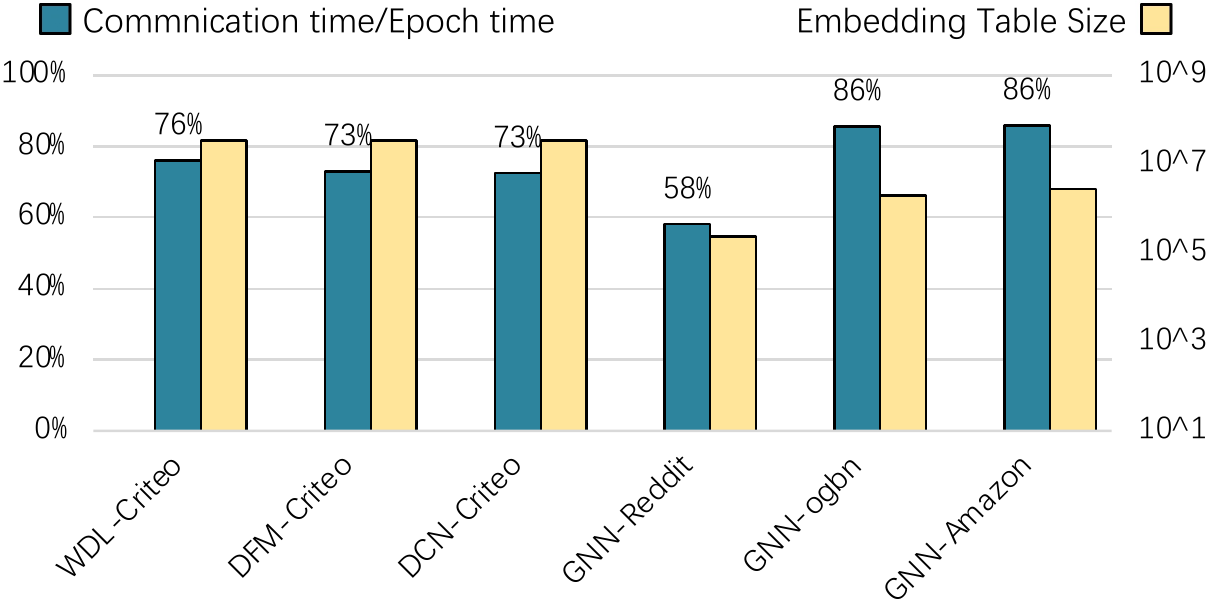}
\vspace{-4mm}
\caption{Large embedding model workloads on TensorFlow}
\vspace{-4mm}
\label{fig:workload}
\end{figure}

\vspace{-3mm}
\subsection{Problems and Opportunities}

\paragraph{\textbf{Communication cost}}
Large-scale embedding models suffer communication bottlenecks from both dense parameters and sparse parameters. PS~\cite{DBLP:conf/sigmod/Zhang0SYJM19} is often suitable for the sparse communication based on the embedding lookup operations. AllReduce~\cite{DBLP:conf/sigmod/MiaoNSYJM021} is highly optimized for the dense communication across GPUs (e.g., NCCL). But it has to degenerate to the inefficient AllGather primitive for sparse communication.
Considering the difference in the sparsity of model parameters, Parallax~\cite{DBLP:conf/eurosys/KimYPCJHLJC19} proposes a hybrid communication architecture that combines PS and AllReduce to transfer sparse and dense parameters respectively. Kraken~\cite{DBLP:conf/sc/XieRLYXWLAXS20} follows the hybrid architecture and optimizes the embedding memory usage. HugeCTR~\cite{hugectr} is NVIDIA’s high-efficiency GPU framework designed for recommendation systems on multiple GPUs, but it is memory restricted since all embedding parameters must be maintained within GPUs.

In general, sparse embeddings dominate the communication bottleneck for large-scale embedding model training. We evaluate the distributed training efficiency of TensorFlow on six popular high-dimensional (almost up to $10^7$) embedding workloads on a single worker, including both click-through rate prediction and graph representation learning. Only the embedding table is deployed on the remote PS. We use a small embedding size $D=32$ and the network bandwidth is 1 Gbps. Figure~\ref{fig:workload} summarizes the time occupation of data transfer and the number of parameters over different embedding models and datasets. Clearly, across all the models and datasets, communication takes much longer time than computation. 
This situation will become more common at modern web companies as the embedding model is still growing in industrial applications (e.g., $|S|=10^9$ to $10^{11}$ in~\cite{DBLP:conf/sc/XieRLYXWLAXS20,DBLP:conf/mlsys/ZhaoXJQDS020} and $D=10^4$ in~\cite{DBLP:journals/corr/DaiOL15}).
Meanwhile, we also find the following characteristics that provide us opportunities
to further improve the performance of embedding model training.

\begin{figure}[t]
\centering
\includegraphics[width=1\linewidth, height=0.33\linewidth]{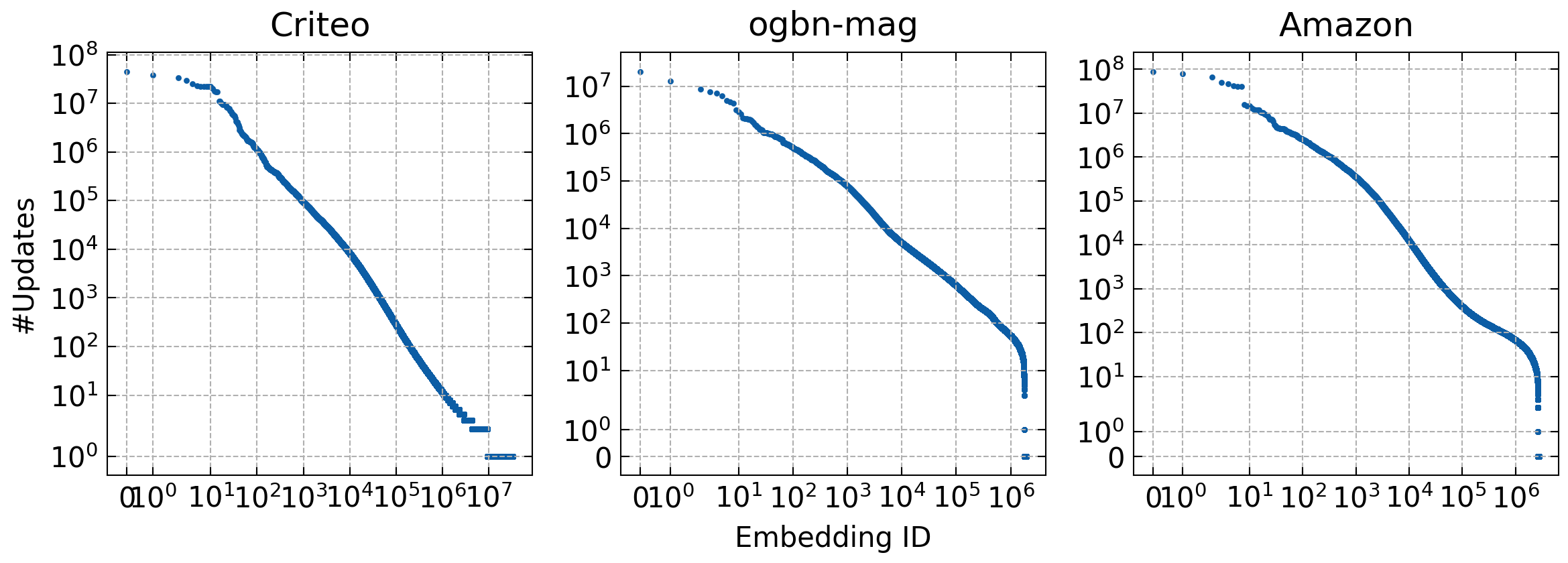}
\vspace{-6mm}
\caption{Embedding popularity skewness}
\vspace{-4mm}
\label{fig:powerlaw}
\end{figure}

\vspace{-1mm}
\paragraph{\textbf{Skewness}} Figure~\ref{fig:powerlaw} illustrates the skew distribution of embedding update frequency on some popular workloads, including click-through rate prediction (i.e., Criteo), citation network (i.e., ogbn-mag), and product co-purchasing network (i.e., Amazon). 
The top 10\% popular embeddings on Criteo could account for 90\% total number of updates. The observation motivates us to reduce the communication by caching frequently updated embeddings in limited local memory. 
Existing research provides evidence that parameter updates from various embedding models exhibit a universal skewed distribution~\cite{9261124}, such as recommendation models~\cite{DBLP:journals/tmis/AdomaviciusZ12,DBLP:conf/www/KangCCYLHC20,DBLP:conf/recsys/ZhaoCCJBBC18}, LDA topic models~\cite{DBLP:conf/eurosys/KimHLZDGX16,xing2016strategies,DBLP:journals/pvldb/YuCZS17,DBLP:journals/dase/HeLZTHD20} and graph learning models~\cite{DBLP:journals/pvldb/LowGKBGH12,DBLP:journals/chinaf/PengLYGWLWR20}.

\vspace{-1mm}
\paragraph{\textbf{Robustness}}
Introducing cache raises the problem of ensuring consistency in the presence of writes. Existing embedding models fall into the category of iterative convergent algorithms which start from a randomly chosen initial point and converge to optima by repeating iteratively a set of procedures. Such iterative convergent process has been shown robust to bounded amount of inconsistency and still converge correctly~\cite{DBLP:conf/icml/LianZZL18}.
This property allows frameworks to improve system performance by relaxing cache consistency models and reading from local (out-of-date) caches.

\begin{figure}[t]
\centering
\includegraphics[width=0.85\linewidth]{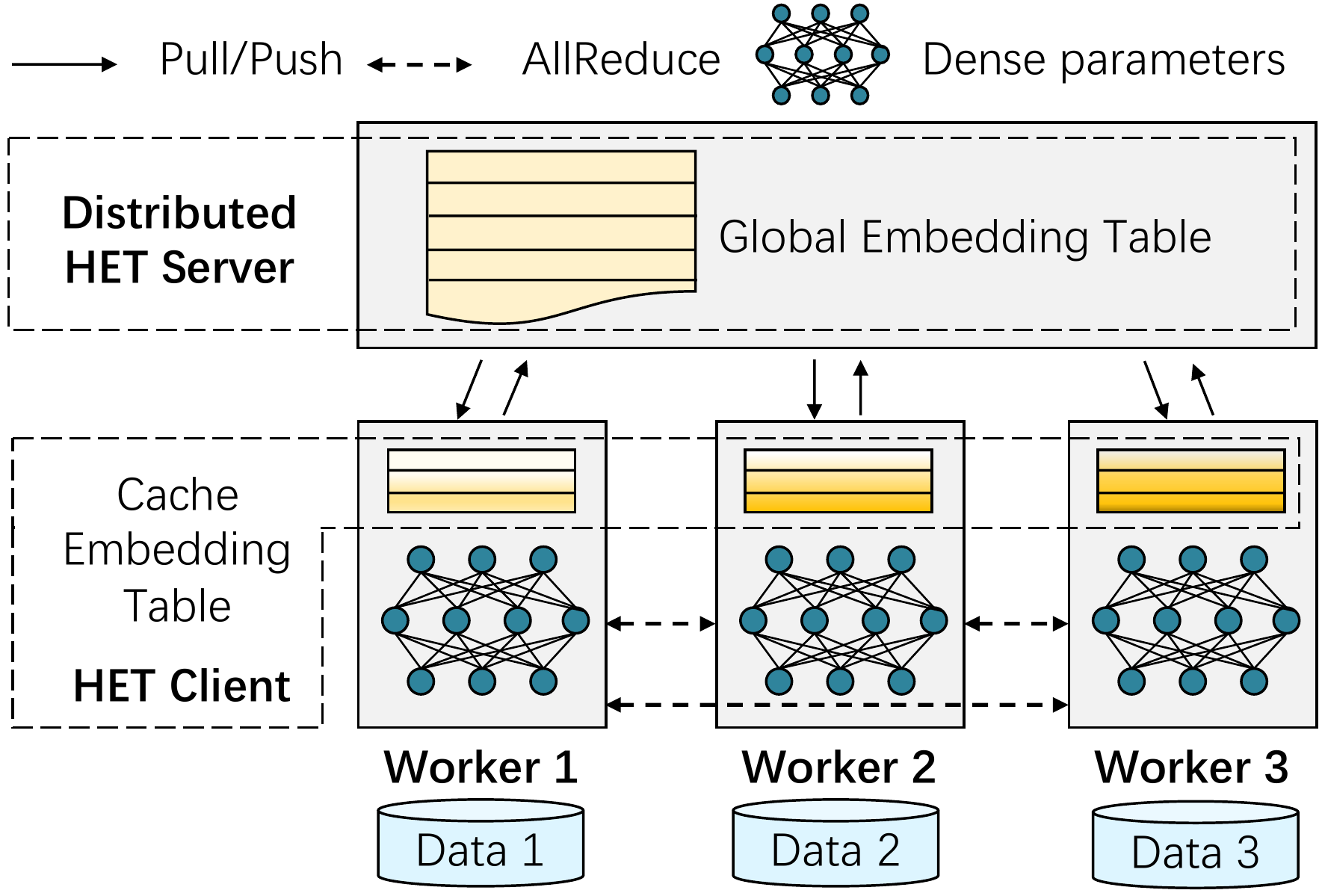}
\vspace{-4mm}
\caption{System architecture of \name}
\vspace{-4mm}
\label{fig:arch}
\end{figure}

\section{\name Design}
\label{sec:design}

This section describes a novel system solution to
scale embedding models by exploiting cache.
Figure~\ref{fig:arch} provides an overview of our system, which distributes the training data into multiple workers. Each worker holds a replication of dense model parameters and uses AllReduce for gradients synchronization during training. 
\name organizes shared embedding parameters as tables.
The whole embedding parameters are stored in the global embedding table on the \name server. 
The client is responsible for the management of the local cache through communicating with servers to control the inconsistency between the local and the global embedding table.
Below, we first introduce the design of cache management and read/write protocols. Then we analyze the cache consistency and model convergence guarantees enforced by our system.

\subsection{Cache Management}\label{sec:model}
The embeddings are organized in a collection of rows in the \textbf{embedding table}. Each embedding represents a sparse feature ID denoted by a unique key $k$. In the global embedding table, each global embedding $\mathbf{x}_k$ records a global Lamport clock~\cite{DBLP:journals/cacm/Lamport78} $\mathbf{x}_k.c_g$ indicating the total number of updates on this embedding. The servers storing the global embedding table act as the cache coordinator.
Each worker can cache a small subset of the embedding table locally.  Instead of recording the ``clock time'' of each client, \name supports fine-grained timing to coordinate cache synchronization at per-embedding basis.
In the cache embedding table, each local embedding $\mathbf{x}_k^i$ on the $i$-th worker records two clocks. (1) The \emph{start clock} $\mathbf{x}_k^i.c_s$ denotes the observed global clock when the last time the embedding $\mathbf{x}_k$ was fetched from server to worker $i$. (2) To guarantee
``read-my-updates'' for a worker $i$, we also increases the \emph{local clock} time $\mathbf{x}_k^i.c_c$ of the embedding $\mathbf{x}_k$ by one once it is updated in a iteration on that worker.

Our system provides operations on the embedding table to manage the cache. Listed below are the core methods of the \name client library. We illustrate some of these interfaces in the execution flow for one iteration in Figure~\ref{fig:interface}.

\begin{figure}[t]
\centering
\includegraphics[width=0.9\linewidth]{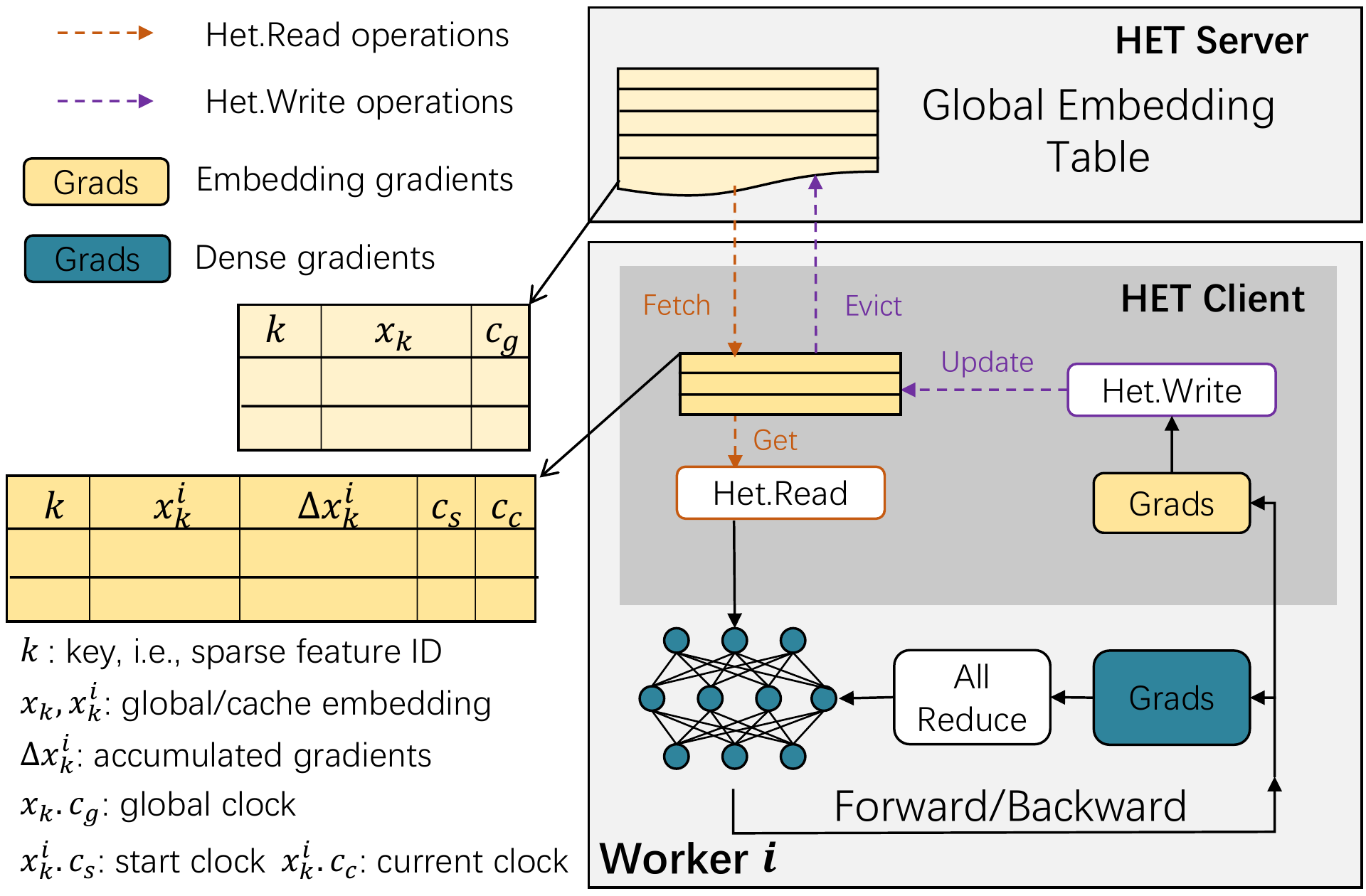}
\vspace{-4mm}
\caption{Workflow of \name}
\vspace{-2mm}
\label{fig:interface}
\end{figure}

\texttt{\sname.Cache.Fetch(key):} This operation directly reads the embedding indexed by \texttt{key} from the global embedding table on the server. For the fetched embedding $\mathbf{x}_k^i$, the start clock $\mathbf{x}_k^i.c_s$ and local clock $\mathbf{x}_k^i.c_c$ both are set to be equal to its global clock $\mathbf{x}_k.c_g$.

\texttt{\sname.Cache.Evict(key):} If input parameter \texttt{key} is provided, this operation finds and evicts the corresponding embedding and push the accumulated gradients and local clock $\mathbf{x}_k^i.c_c$ to the server.
The server receives the accumulated gradients, applies them on the corresponding entry of global embedding table, and synchronizes the global clock $\mathbf{x}_k^i.c_g= \max(\mathbf{x}_k^i.c_g, \mathbf{x}_k^i.c_c)$.
If \texttt{key} is not provided, this operation tries to evict the overflowed embeddings selected by certain cache policies (e.g., LRU, LFU) to prevent the cache table from exceeding the size limitation. We further discuss the selection of cache policies in Section~\ref{sec:cache_policy}.

\texttt{\sname.Cache.CheckValid(key):} This operation finds the local embedding $x_k^i$ from the cache according to the \texttt{key} and returns true if its current clock $x_k^i.c_c$ satisfies the following two time-bound conditions: 
(1) the current clock should not be too far ahead of the start clock, i.e., $x_k^i.c_c\leq x_k^i.c_s+s$; (2) the current clock should not be too far behind its global clock, i.e., $x_k.c_g\leq x_k^i.c_c+s$. Here $s$ is a user-defined staleness threshold to determine the cache validity. Since the global clock is recorded on the server, to validate condition (2), we have to send $x_k^i.c_c$ from the worker to the server when the cache hit occurs. Note that, the communication costs in this step are not significant because we only send the clocks, rather than the embedding vectors.

\begin{algorithm}[t]
\caption{\sname Client}
\label{alg:client}
\LinesNumbered
\KwIn{input dataset $\xi$, max iterations $T$, model parameters $\mathbf{x}_0$ and $\mathbf{x}_e$, DL runtime \texttt{DL}}
\KwOut{Trained model $\mathbf{x}_0$ and $\mathbf{x}_e$.}
Initialize model parameter $\mathbf{x}_0$\;
\texttt{\sname.Intialize($\mathbf{x}_e$)}\;
\For{i $\in$ \text{range}($T$)}{
    $\xi_i\gets$sample a mini-batch of data from $\xi$\;
    $K_i\gets$ the unique embedding key set of $\xi_i$\;
    $E_i\gets \texttt{\sname}.\texttt{Read}(K_i)$ \tcc*{Read embeddings}
    $G_i\gets \texttt{DL.Forward/Backward}(\xi_i, E_i, \mathbf{x}_0)$\;
    $\texttt{DL.Update}(G_i(\mathbf{x}_0))$ \tcc*{Locally update dense}
    $\texttt{\sname}.\texttt{Write}(K_i, G_i(\mathbf{x}_e))$ \tcc*{Write embedding}
}
\end{algorithm}

\subsection{Read/write Protocols} 
We briefly introduce the workflow of \name client in Algorithm~\ref{alg:client}. After each worker initializes the model parameters, it repeats the iterations based on the mini-batch SGD algorithm and trains the embedding model. The client extracts the unique keys (i.e., feature ID) from the mini-batch of data and \texttt{Read}s the corresponding embeddings. The DL executor performs forward and backward computation using these model parameters and input data. After that, we could update the dense model parameters and sparse embeddings, respectively.

\texttt{\sname.Read(keys):} Read a set of embedding vectors based on the requested \texttt{keys} as shown in Algorithm~\ref{alg:read}. For each key $k$, the client first checks whether $k$ exists in the cache embedding table (i.e., \texttt{\sname.Cache.Find(key)} in line 3). If not, the client fetches the latest version embedding from the server and adds it into the local cache embedding table temporarily (line 8); If so, the client further checks whether the caching embedding is valid (line 4) and manages to cache up with the server by synchronizations (line 5). 
For those embeddings within the staleness threshold, the client directly reads from the local embedding table (i.e., \texttt{\sname.Cache.Get(key)}).

\texttt{\sname.Write(keys, gradients)}: Writing back the embedding gradients as shown in Algorithm~\ref{alg:write}. Our cache embedding table allows \textbf{stale-writes} to reduce the communication cost between client and server. Since all the embeddings with \texttt{keys} have been loaded in the cache embedding table before the forward and backward computation, we could directly write the gradients locally by accumulating them on the corresponding rows of the cache embedding table (i.e., \texttt{\sname.Cache.Update(key, grad)} in line 2).
This enforces the read-my-updates property,
which ensures that the data read by a client contains all
its own updates.
Meanwhile, these cache embeddings should increase their current clocks by 1 (line 3). These accumulated updates could only be written back to server later through the cache eviction operation, which become ``stale'' relative to the global embedding table.

\subsection{Cache Consistency Guarantee}

\begin{algorithm}[t]
\caption{\sname.Read}
\label{alg:read}
\LinesNumbered 
\KwIn{input key set $K$}
\KwOut{Embedding set $E$}
$E\gets\{\}$\;
\For{k $\in$ K}{
    \eIf{\texttt{\sname.Cache.Find(k)}}{
        \If{not \texttt{\sname.Cache.CheckValid(k)}}{
            \texttt{\sname.Cache.Evict(k)} \tcc*{Synchronize}
            \texttt{\sname.Cache.Fetch(k)}  \tcc*{embeddings}
        }
    }{
       \texttt{\sname.Cache.Fetch(k)};
    }
    $E\gets E\ \cup\ $\texttt{\sname.Cache.Get(k)}\;
}
\end{algorithm}

\begin{algorithm}[t]
\caption{\sname.Write}
\label{alg:write}
\LinesNumbered
\KwIn{input key set $K$, the embedding gradients $G$}
\For{k $\in$ K}{
    \texttt{\sname.Cache.Update(k, $G_k$)}\;
    \texttt{\sname.Cache.Clock($k$)} \tcc*{Increase $c_c$ by 1}
}
\texttt{\sname.Cache.Evict}()\;
\end{algorithm}

From the per-embedding perspective, each embedding might exist in multiple cache embedding tables during training. Therefore, the cache consistency guarantee is crucial for the final model quality. Before interpreting the cache consistency model, we first clarify the following lemma on the clock consistency between any two embedding replications in different workers:

\begin{lemma}
 \label{lem:stale}
  For any $\mathbf{x}_k$, let $\mathbf{x}_k^i, \mathbf{x}_k^j$ are its two replicas cached on worker $i,j$, respectively, \name guarantees that:
  \begin{equation}
    \small{
      \forall k, \max_{\forall 0\leq i, j\leq N}\{|\mathbf{x}_k^i.c_c-\mathbf{x}_k^j.c_c|\}\leq 2s.
      }
  \end{equation}
\end{lemma}

\begin{proof}
    For any embedding $\mathbf{x}_k$ at iteration $t$, the replication on the $i$-th worker is denoted by $\mathbf{x}_k^i$.
    Based on the conditions in \texttt{CheckValid(key)}, we have: $\mathbf{x}_k.c_g-s\leq \mathbf{x}_k^i.c_c$ and $\mathbf{x}_k^i.c_c\leq \mathbf{x}_k^i.c_s+s\leq \mathbf{x}_k.c_g+s$. Therefore, for any two different workers $i$ and $j$, the difference between their local current clocks is upper bounded by $2s$.
\end{proof}

Lemma~\ref{lem:stale} formally describes the clock-bounded guarantee at per embedding basis.
This guarantee enforces the following consistency model across embedding replications in multiple caches:

\begin{definition}[\textbf{Per-embedding clock bounded consistency}]
For any embedding $\mathbf{x}_k$, the consistency model guarantees that a worker $i$ sees the updates of any other worker $j$ on embedding $\mathbf{x}_k$  in the range of $[0, \mathbf{x}_k^j.c_c-2s]$.
\end{definition}

It is worth pointing out that Lemma 1 and the per-embedding clock bounded consistency only describe the observable embeddings.
If a worker fetches a key, makes an update, and never sees that key again, while other workers continue to see that key. The embedding in that worker is going to be evicted
when the cache capacity is full, and the update would be written back to the server. In the corner case, it might remain in the cache until the model completes the training process, we could simply ignore that update due to the robustness of iterative convergent algorithms.

\subsection{Convergence Analysis}
Recall that we differ from traditional SSP in the following aspects to improve performance: (1) SSP is not aware of the presence of skew access and provides bounded staleness at the coarse granularity measured by worker clocks, whereas we provide bounded staleness at the fine granularity of individual embedding clocks considering access heterogeneity. (2) SSP assumes a write-through cache so that the server is up-to-date, whereas we adopt a write-back-with-stale server to improve write performance. Given these key differences, we present a new theoretical analysis of the proposed consistency model, instead of reusing that of SSP.
We summarize our proof results showing that our algorithm is guaranteed to converge under our consistency model. The details of the assumptions and proofs are in Appendix A~\cite{hetappendix}.
We decompose the $d$ model parameters into two categories, including the dense parameters $x_0$ and the sparse parameters of $\mathbf{x}_e$ containing $m$ embedding vectors $(x_1, x_2, \ldots, x_m)$. We denote $x_i$ and $\nabla_i f(\mathbf{x})$ as the $i$-th component of $\mathbf{x}$ and $\nabla f(\mathbf{x})$, respectively. Clearly, $\mathbf{x} = (x_0, x_1, \ldots, x_m)$ and $\nabla f = (\nabla_0 f, \nabla_1 f, \cdots, \nabla_m f)$. We have the following theorem:

\begin{theorem}[Global Convergence Rate]
\label{thm:convergence}
  Consider an arbitrary objective function $f$, under the per-embedding-clock-bounded consistency model and and certain assumptions,
  Given the success parameter $\epsilon > 0$, a constant learning rate value 
  \begin{equation}
  \small{\eta \leq \min( \frac{\sqrt{ \epsilon } } {4\sqrt {3}sLMB}, \frac{\sqrt{ \epsilon } } {4\sqrt {L}sMB} , \frac{\epsilon}{ 12 M^2B^2 L})\nonumber}
  \end{equation}
  and  $T = \Theta \left( \frac{ f ( x(0) ) - f_{\text{inf}}}{\epsilon \eta} \right)$ iterations, for worker $j$, we are guaranteed to reach some iterate $\mathbf{x}^j(t^\star)$ with $1 \leq t \leq T$ such that \small{$$\mathbb{E} \| \nabla f(\mathbf{x}^j(t^\star)) \|^2 \leq \epsilon.$$}
\end{theorem}

\begin{figure*}[t]
\centering
\subfigure[WDL-Criteo]{
\scalebox{.315}{
\includegraphics[width=1.0\linewidth]{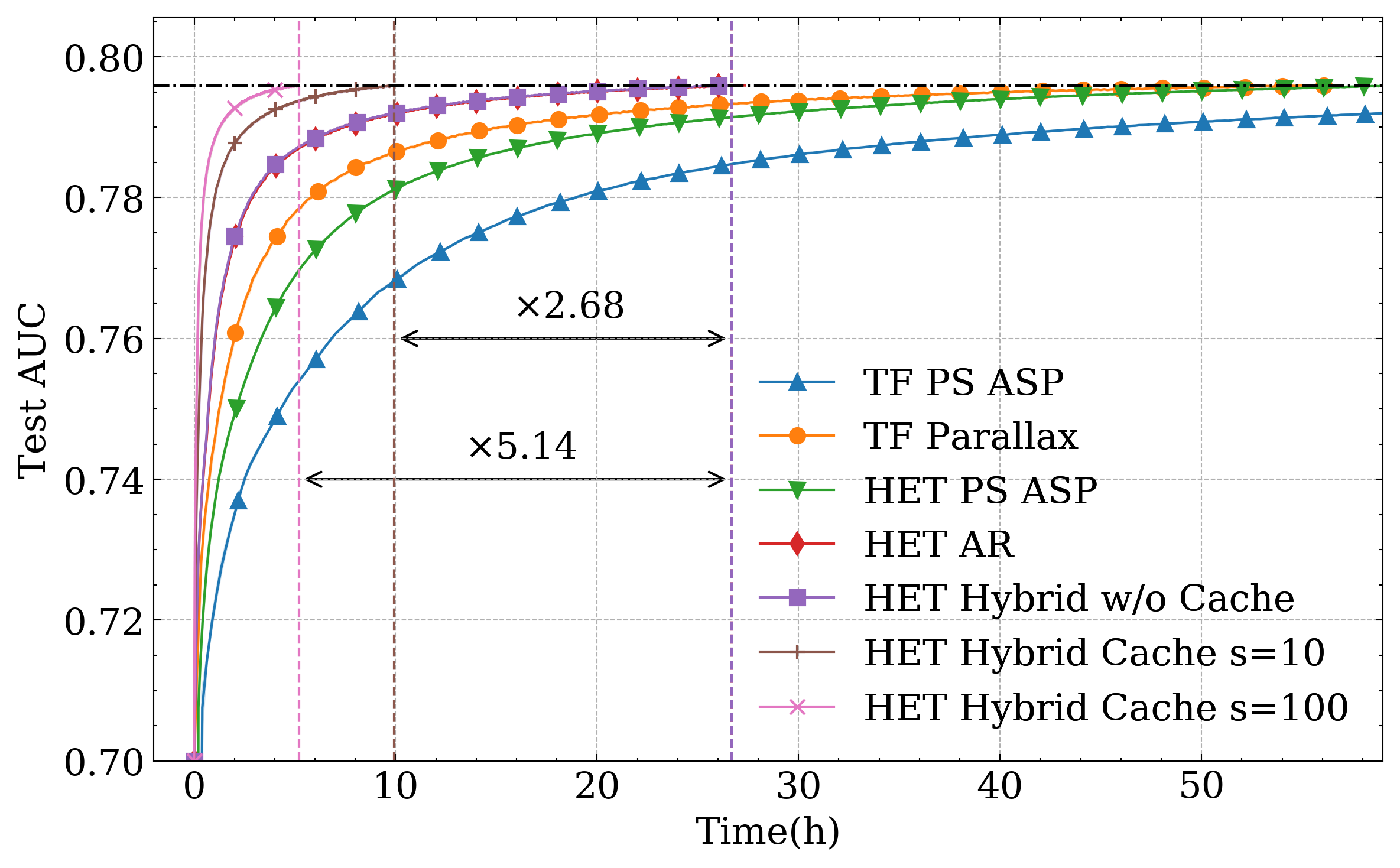}
}
\label{fig:wdl_conv}
}
\subfigure[DFM-Criteo]{
\scalebox{.315}{
\includegraphics[width=1.0\linewidth]{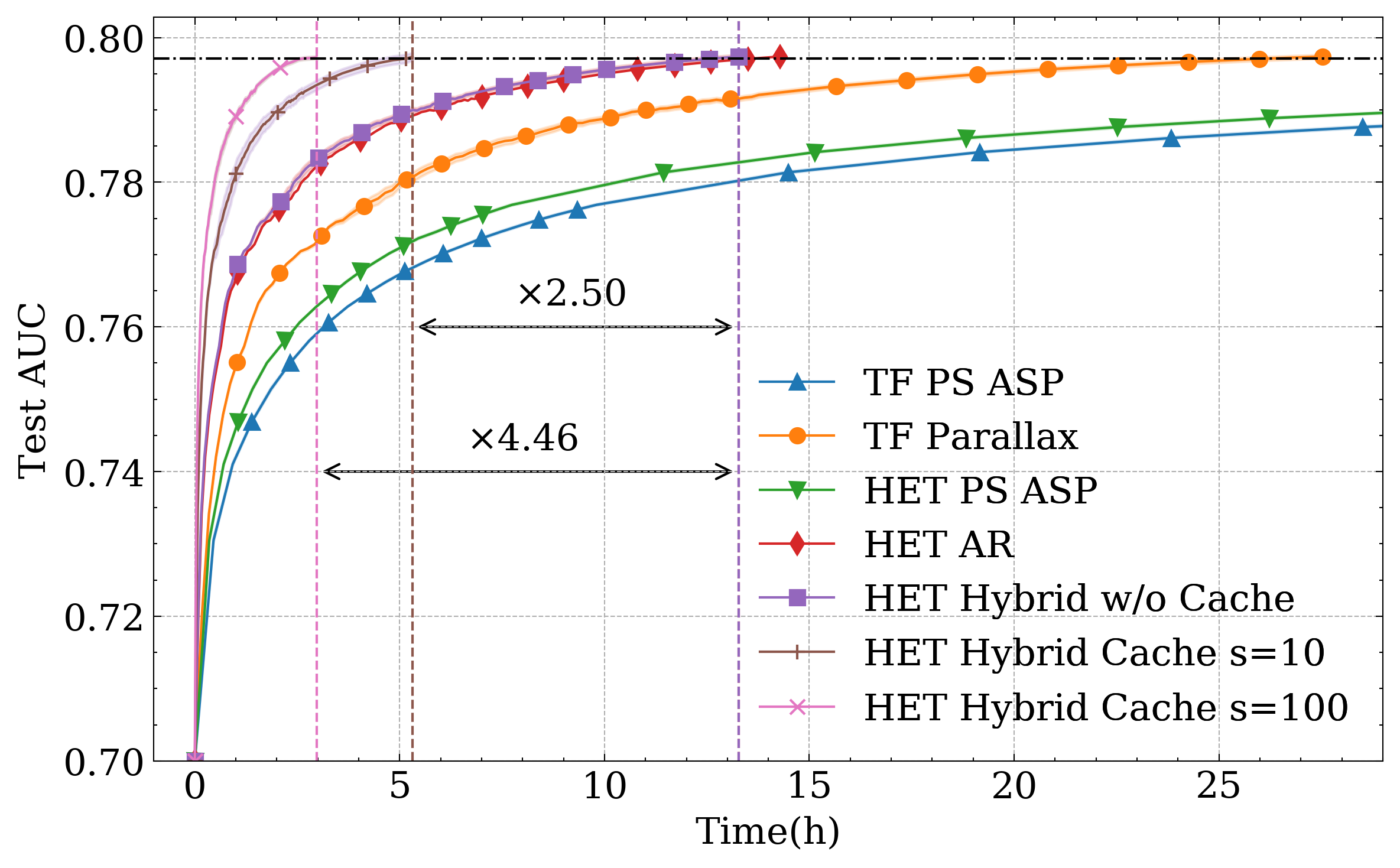}
}
\label{fig:dfm_conv}
}
\subfigure[DCN-Criteo]{
\scalebox{.315}{
\includegraphics[width=1.0\linewidth]{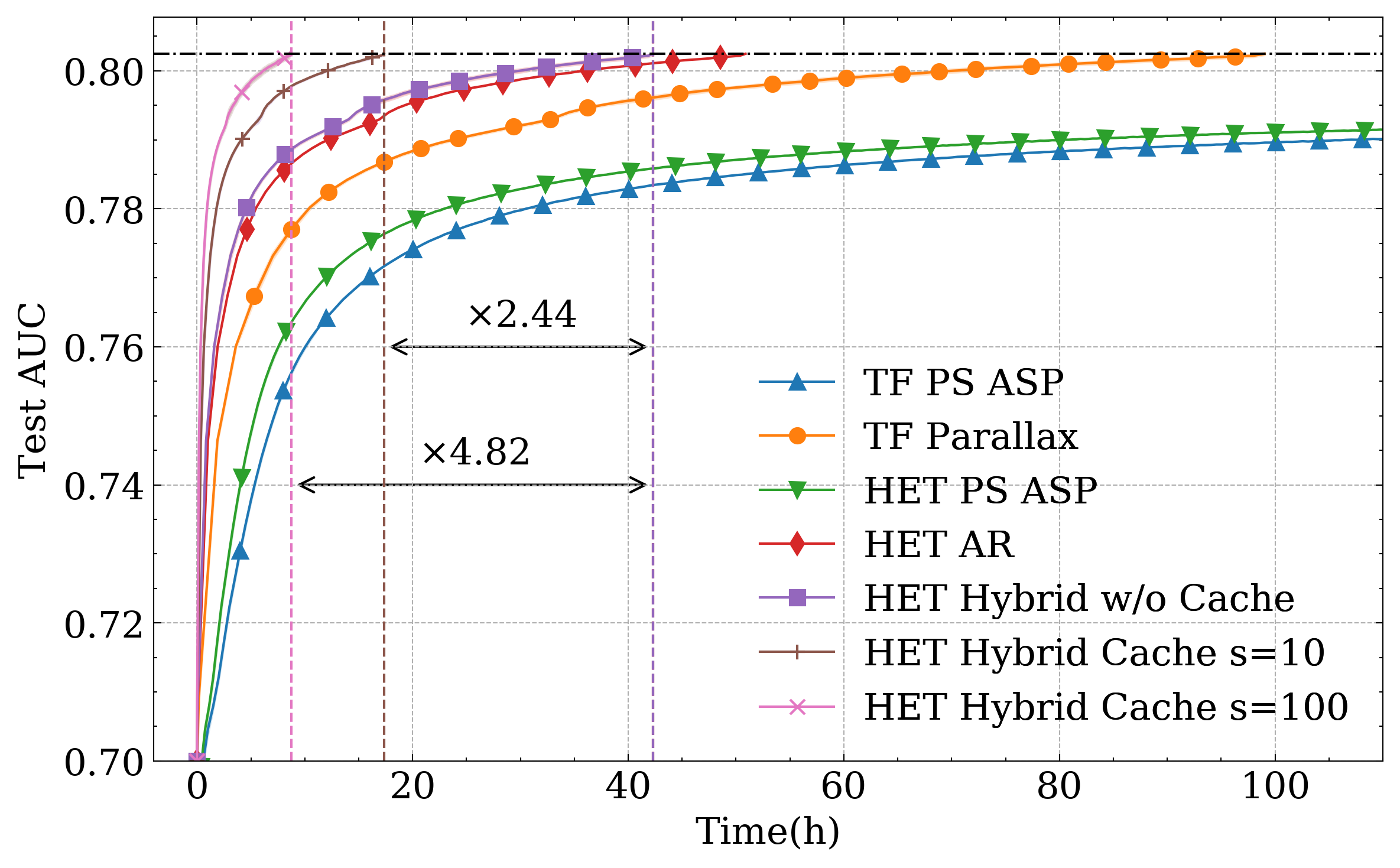}
}
\label{fig:dcn_conv}
}
\vspace{-2mm}
\subfigure[GNN-Amazon]{
\scalebox{.315}{
\includegraphics[width=1.0\linewidth]{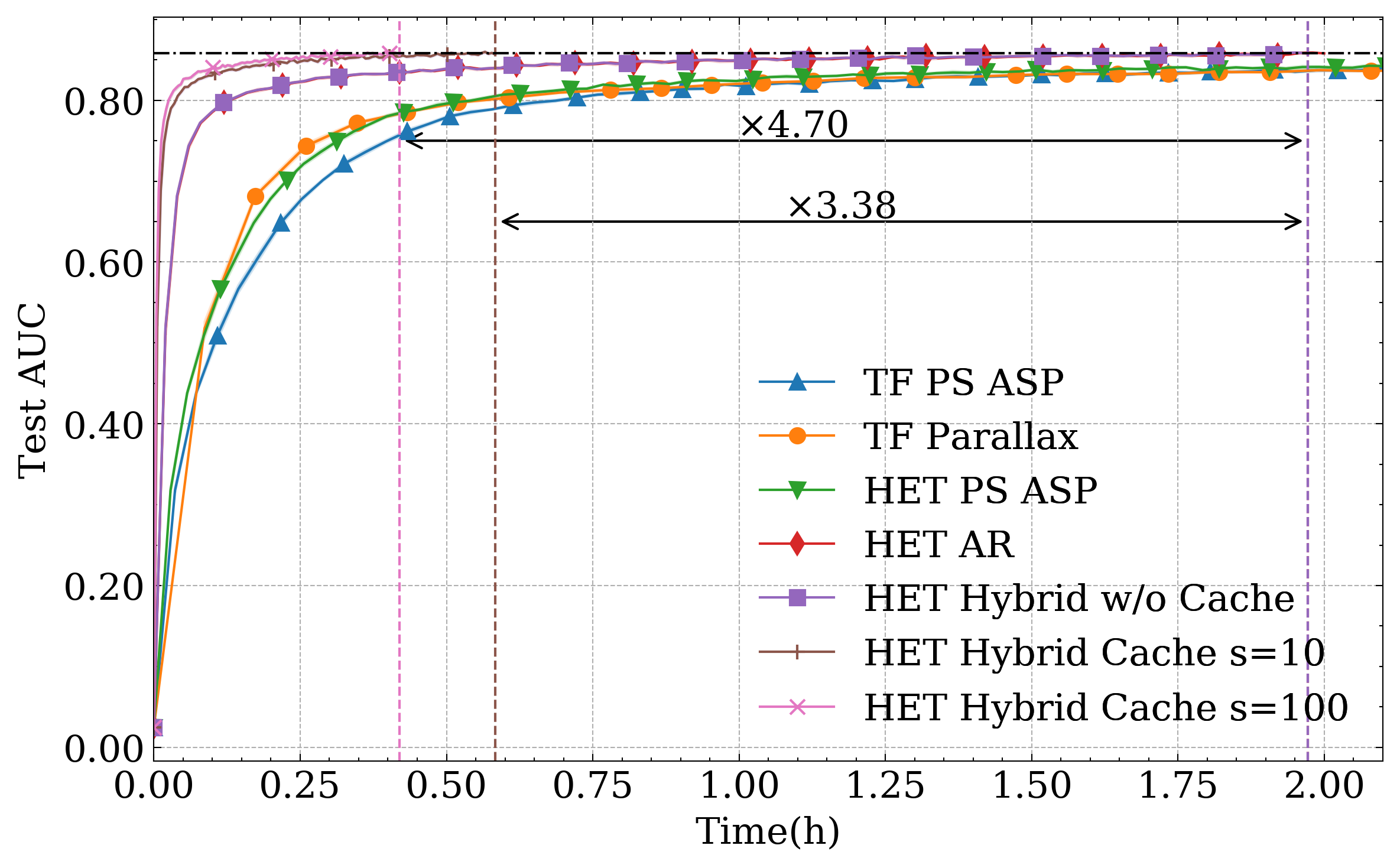}
}
\label{fig:amazon_conv}
}
\subfigure[GNN-Reddit]{
\scalebox{.315}{
\includegraphics[width=1.0\linewidth]{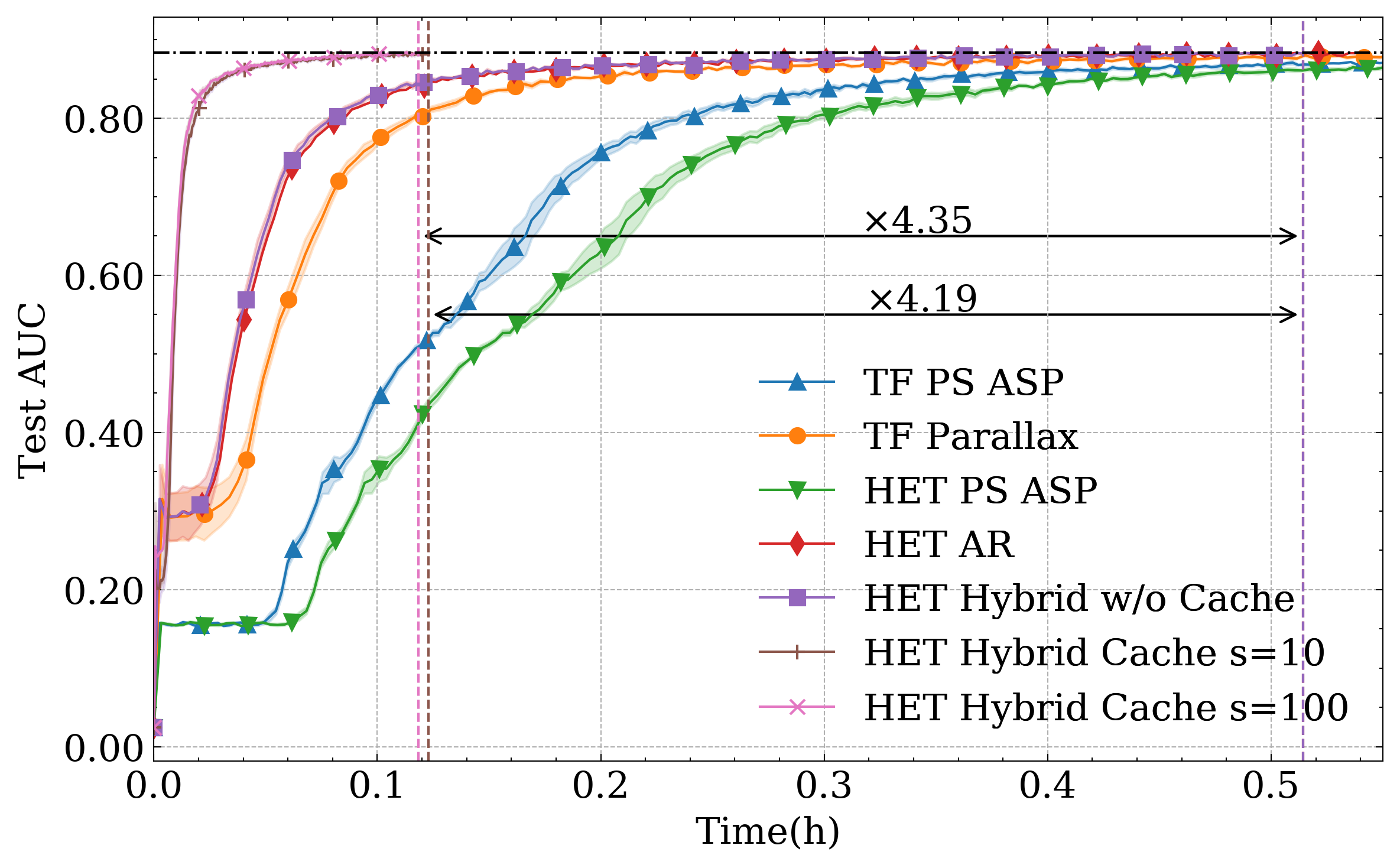}
}
\label{fig:reddit_conv}
}
\subfigure[GNN-obgn-mag]{
\scalebox{.315}{
\includegraphics[width=1.0\linewidth]{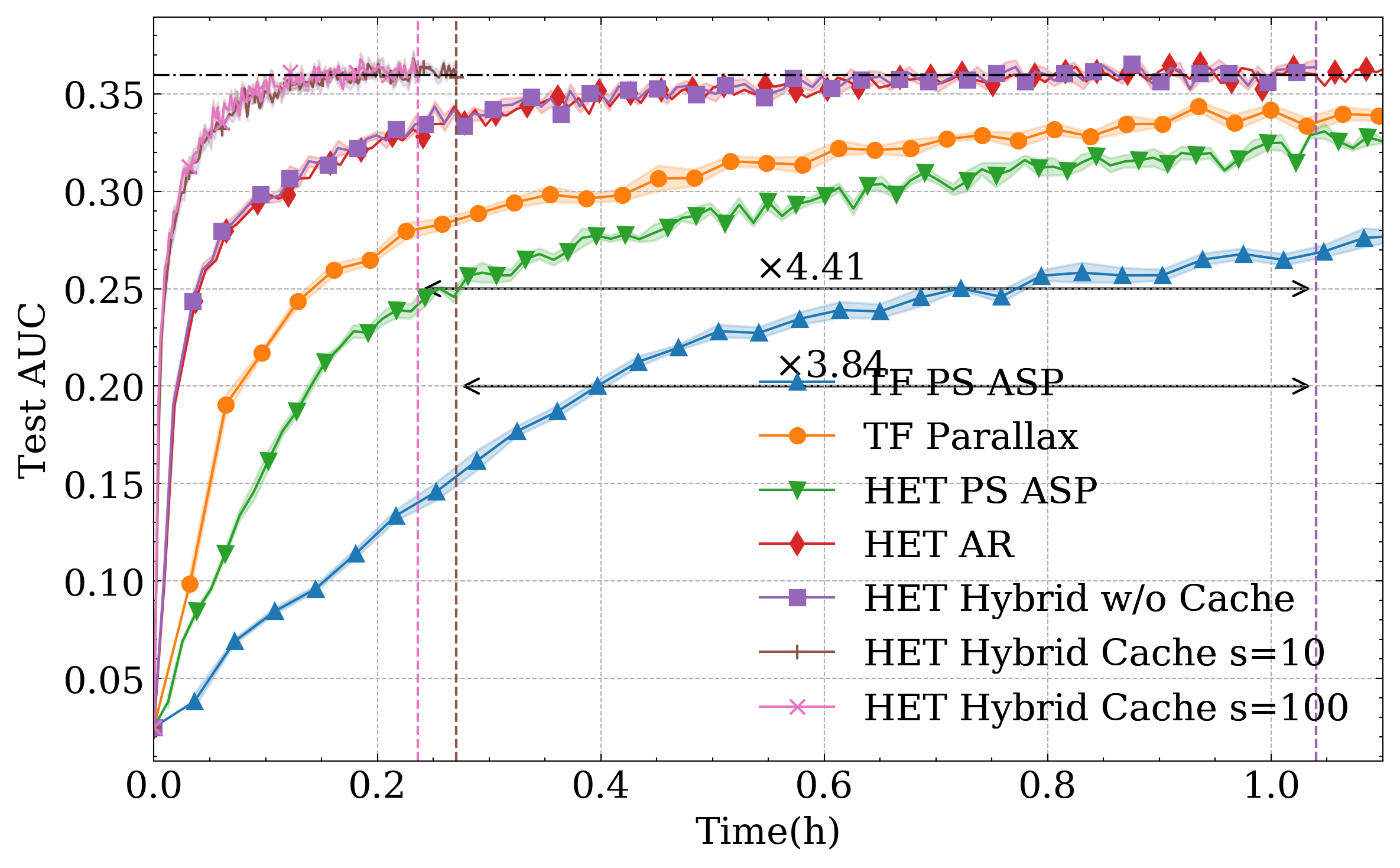}
}
\label{fig:obgn_conv}
}
\vspace{-2mm}
\caption{Convergence performance comparison.}\label{fig:conv}
\vspace{-2mm}
\end{figure*}

\section{\name Implementation}
\name's implementation is built on Hetu\footnote{\url{https://github.com/PKU-DAIR/Hetu/}}, a DL system consists of 14.5K LOC in C/C++/CUDA with a Python dataflow front-end (20.7K LOC).
It is easy to extend our cache embedding mechanism to other DL systems by replacing the DL runtime (e.g., TensorFlow, PyTorch, MXNet). Taking TensorFlow (TF) as an example, we could first replace the native TF parameter server with our HET server to store the global embedding table. Then we could implement an embedding variable inheriting from TF, and encapsulate the lookup/update operations with HET client interfaces. We leave the extension as our future work.

In our hybrid communication architecture, AllReduce is implemented by MPI~\cite{10.5555/898758} and NCCL~\cite{nccl}; the key components -- \name client and server, are developed based on PS-Lite\cite{pslite}, a lightweight implementation of PS interface. Currently, we implement the embedding table in C++ and store the cache embedding table in the limited DRAM (e.g., 12 GB) of each worker in our experiments. We manage to improve the overall performance by leveraging the following implementation optimizations. 

\vspace{-2mm}
\subsection{Asynchronous Communication Invocation}\label{section:comp_graph}
Similar to TensorFlow\cite{DBLP:conf/osdi/AbadiBCCDDDGIIK16}, we use a static computation graph abstraction to organize all the operations in \name. All operators implemented by GPU kernels are scheduled into the GPU stream~\cite{9261124}. These operators will be launched and executed asynchronously to avoid blocking the CPU execution. In \name, the communication operations (e.g., AllReduce, Fetch and Evict) are also treated similarly to overlap with computations. To ensure dependencies between computation and communication, we borrow the idea of CUDA event from GPU. We asynchronously launch communication requests and record corresponding events to synchronize when the updated parameter should be used in the next iteration.

\vspace{-2mm}
\subsection{Message Fusion}
Combining \texttt{Pull} (model parameters) and \texttt{Push} (model gradients) operations in parameter server architecture is a common technique~\cite{DBLP:conf/sosp/PengZCBYLWG19} for performance improvement.
It is easy to implement for traditional dense parameters. As for sparse parameters, especially for embedding tables, it is non-trivial since the sparse access property of embedding models.
To combine the cache eviction and fetching, we need to pre-fetch the next mini-batch of data in advance to inform the embedding indices. Besides, we also remove the duplicate keys in the request of each mini-batch to reduce redundant communication costs.

\vspace{-2mm}
\subsection{Cache Strategies}
\label{sec:cache_policy}
The goal of our cache strategy is to maximize the embedding lookup hit rate within a restricted amount of memory. The effectiveness of the cache is decided by two major factors: the query frequency from local workers and the length of its expiration period. The latter one is affected by other workers' workload and is hard to predict, so we only focus on optimizing the first objective, which is to cache the most frequently used embeddings by the local workers (e.g., LFU and LRU policy). Due to the high maintenance cost incurred by LFU, we provide a light-weighted version of LFU. When the frequency of an embedding is high enough, it will be assigned a direct access index, bypassing the cost of frequency maintenance. Under the same workload, it could have a similar miss rate as the original LFU while retaining a significantly small run-time cost.

\section{Experiments}

\paragraph{\textbf{Baselines}} In this section, we compare our prototype system with two state-of-the-art systems: TensorFlow (TF)~\cite{DBLP:conf/osdi/AbadiBCCDDDGIIK16} and Parallax~\cite{DBLP:conf/eurosys/KimYPCJHLJC19}. 
To alleviate the concerns on the difference from the system backbones and implementations, we implement three auxiliary baselines over our system named by \name PS (ASP), \name AR (AllReduce) and \name Hybrid (w/o Cache). 
\name PS follows the ASP algorithm in TensorFlow and each worker pushes its updates to the server without waiting for the others. \name Hybrid keeps the hybrid communication architecture in \name but removes the cache embedding table. 
All of these three baselines are sharing the same computation kernels and communication optimizations (i.e., overlapping, pre-fetching) as \name (denoted by \name Hybrid Cache or \name Cache in the following).

\vspace{-2mm}
\paragraph{\textbf{Datasets and models}}
We select two categories of representative embedding model workloads, including the deep learning recommendation model (DLRM) and the graph neural network (GNN). We use three industrial DLRM models consisting of Wide \& Deep (WDL)~\cite{DBLP:conf/recsys/Cheng0HSCAACCIA16}, DeepFM (DFM)~\cite{DBLP:conf/ijcai/GuoTYLH17} and Deep \& Cross (DCN)~\cite{DBLP:conf/kdd/WangFFW17}. They are evaluated on a popular recommendation dataset, Criteo~\cite{criteo}, which is also the largest standard benchmark in MLPerf~\cite{mlperf}. There could be more than \textit{one trillion model parameters} (i.e., $10^{12}$ floats) from the embedding table when we set $D=4096$ on Criteo.

The second category is GNN~\cite{DBLP:journals/chinaf/GuoQX021,miao2021lasagne,DBLP:conf/kdd/MiaoGZHLMRRSSWW21} model and we select the most popular GraphSAGE~\cite{DBLP:conf/nips/HamiltonYL17} for large graphs. We evaluate on node classification tasks and adopt several graph datasets with different scale: Reddit\cite{hamilton2017inductive}, Amazon\cite{clustergcn} and ogbn-mag\cite{DBLP:conf/nips/HuFZDRLCL20}. More details about the datasets and models are in the Appendix B~\cite{hetappendix}.

\begin{figure}[t]
\centering
\centering
\subfigure[1 Gb cluster]{
\scalebox{1}{
\includegraphics[width=1.0\linewidth]{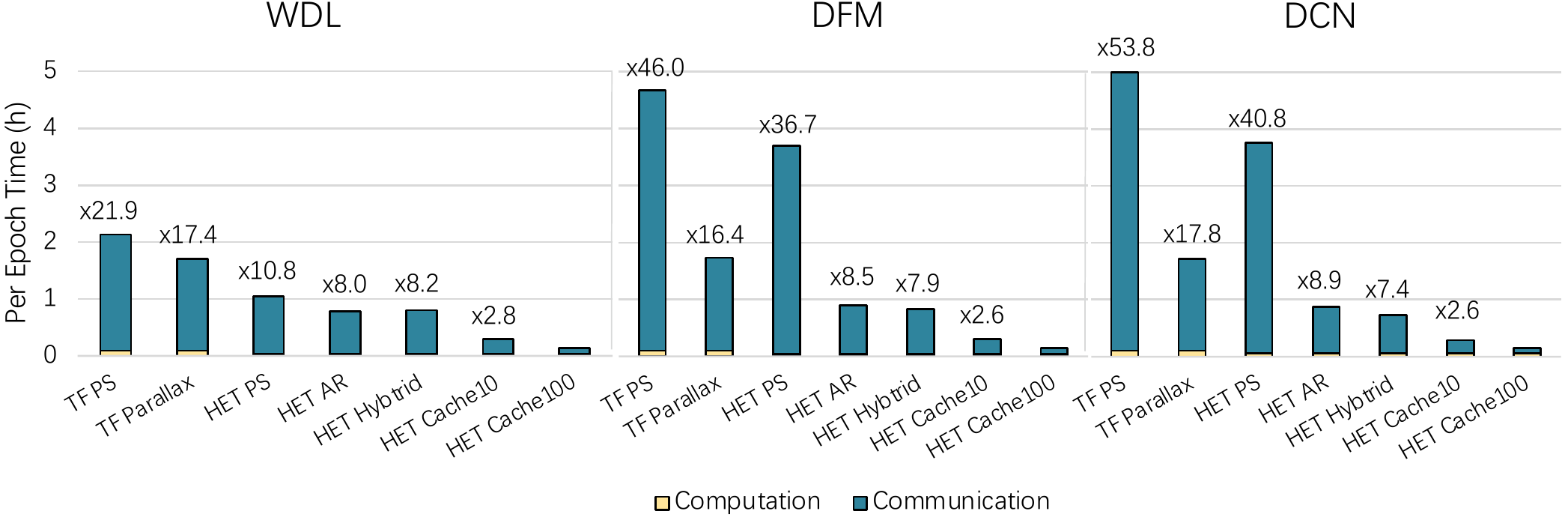}
}
\label{fig:recepoch_a}
}
\subfigure[10 Gb cluster]{
\scalebox{1}{
\includegraphics[width=1.0\linewidth]{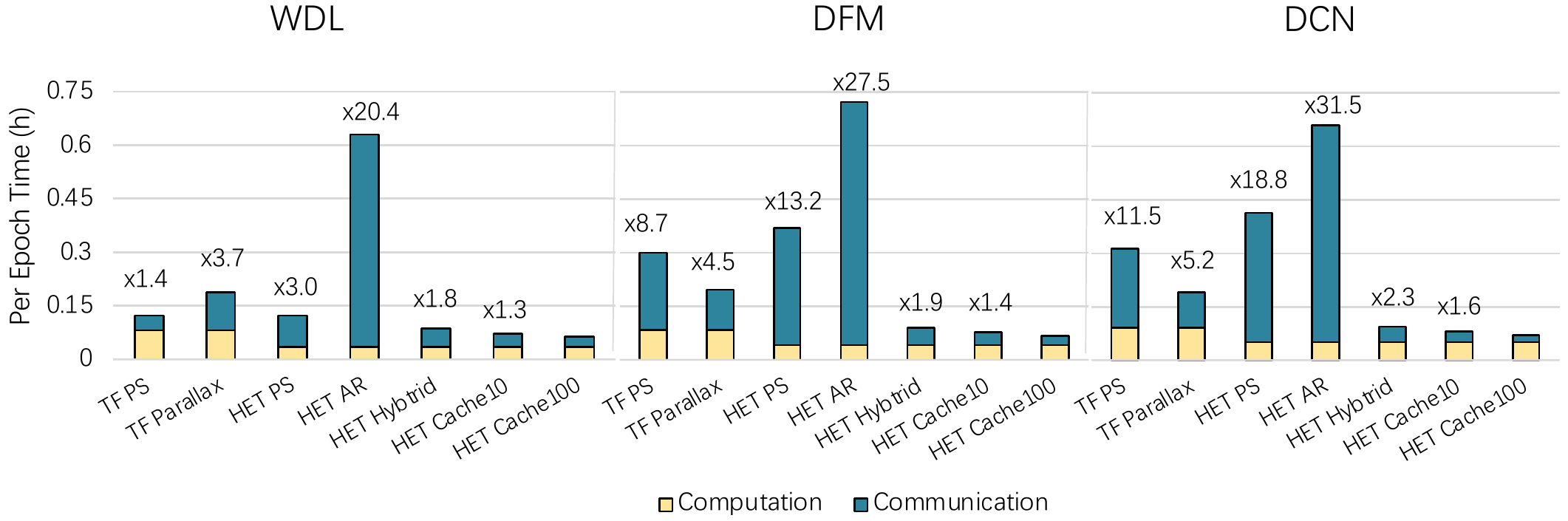}
}
\label{fig:recepoch_b}
}
\vspace{-4mm}
\caption{Per epoch time and the speedup of communication time on DLRM tasks.}
\vspace{-4mm}
\label{fig:recepoch}
\end{figure}

\paragraph{\textbf{Experimental setting}}
We implement all of these models in TensorFlow 1.15 and select SGD optimizer with the batch size of 128; the learning rate is selected from [0.001, 0.01, 0.1] by grid search. We have two GPU clusters for our evaluation. In cluster A, each node is equipped with an Nvidia RTX TITAN 24 GB card supporting PCIe 3.0 and 12 GB DRAM for cache (we temporarily improve the DRAM size to 48 GB in the last model scalability experiment). While the servers are in a CPU cluster and each node has two Intel Xeon Gold 5120 CPUs and 376 GB DRAM. The servers and workers clusters are connected by a 1 Gbit Ethernet. Cluster B has similar configurations, but the GPU is replaced by Nvidia V100 and the network is improved to 10 Gbit Ethernet.
The testing AUC thresholds of convergence are set to be around 80\% for the Criteo dataset, as reported in~\cite{DBLP:conf/recsys/Cheng0HSCAACCIA16}. 
For the GNN datasets, due to the customized feature engineering step (e.g., introducing sparse features), we manually set the termination point by predefined values.
All experiments are executed five times, and the averaged results are reported.

\subsection{End-to-end Comparison}
In this section, we first provide end-to-end comparison experiments with the baselines. These experiments are evaluated on 8 workers and 1 remote server. The size of the cache embedding table is set to be 10\% size of the global cache embedding table (listed in Figure~\ref{fig:workload}). We set $D=128$ in the following experiments and tune it in Sec.~\ref{sec:scala}.

\textit{\textbf{Convergence efficiency.}}
Figure~\ref{fig:conv} shows the convergence curves on six different workloads on cluster A. TF PS and HET PS follow the ASP algorithm and cannot converge to the target thresholds in these workloads.
We provide \name with different staleness thresholds $s=10$ and $100$. 
As we can see, our system always outperforms the other baselines on all tasks. When $s=100$, benefited from our cache embedding table mechanism, we could achieve around 4.36-5.14$\times$ speed up compared to \name Hybrid. Compared to $s=10$, it is natural that using a larger $s$ could alleviate more communication costs. The variances regions are quite negligible, which verifies the convergence stability of our methods.
Table~\ref{tab:end2endtime} illustrates the end-to-end convergence time and our \name achieves 6.37-20.68$\times$ speedup compared to TF Parallax. PS-based ASP methods are not listed because they cannot achieve the convergence thresholds. Note that, in our system implementation, we make a unique operation for the keys before the embedding communication operations to avoid redundant embedding transferring costs. However, the GNN-Reddit workload is quite special. It only has node-id embeddings and all embeddings in a mini-batch of data samples are always unique naturally. Therefore, in this case, the costs of unique operation outweigh the benefits, making HET PS ASP slower than TF PS ASP.

\begin{table}[t]
\centering
\small
\caption{End-to-end convergence efficiency comparison.}
\label{tab:end2endtime}
\vspace{-2mm}
\scalebox{0.85}{
\begin{tabular}{l|rrrc}
\toprule 
\makecell{Convergence\\time (h)} & TF Parallax & \name Hybrid & \makecell{\name Cache\\$s=10$} & \makecell{\name Cache\\$s=100$}\\
\midrule 
WDL-Criteo & 56.402 ($\times$10.86) & 26.668 ($\times$5.14) & 9.938 ($\times$1.91) & \textbf{5.193}\\
DFM-Criteo & 27.529 ($\times$9.24) & 13.273 ($\times$4.46) & 5.314 ($\times$1.78) & \textbf{2.978}\\
DCN-Criteo & 99.023 ($\times$11.29) & 42.296 ($\times$4.82) & 17.341 ($\times$1.98) & \textbf{8.770}\\
\midrule
GNN-Amazon  & 8.667 ($\times$20.68) & 1.972 ($\times$4.71) & 0.583 ($\times$1.39) & \textbf{0.419}\\
GNN-Reddit  & 0.752 ($\times$6.37) & 0.514 ($\times$4.36) & 0.123 ($\times$1.04) & \textbf{0.118} \\
GNN-ogbn-mag  & 3.869 ($\times$16.39) & 1.040 ($\times$4.41) & 0.271 ($\times$1.15) & \textbf{0.236}\\
\bottomrule
\end{tabular}}
\vspace{-4mm}
\end{table}

\textit{\textbf{Communication speedup.}}
We also compare the per epoch time on DLRM tasks and provide the following findings. First, through comparative analysis on the per epoch time in Figure~\ref{fig:recepoch_a} and the learning curves in Figure~\ref{fig:conv}, 
we find that \name PS and TF PS follow the same statistical efficiency~\cite{DBLP:conf/sigmod/JiangCZY17} in Figure~\ref{fig:conv} (\name Hybrid and TF Parallax are also the same). Their different convergence speeds come from the backbone optimizations, which verify the correctness of our implementation. Second, PS-based methods often show poor performance compared to hybrid-based methods due to the dense communication. The phenomenon in WDL is not significant because it has fewer dense model parameters than DFM and DCN. Third, these results in Figures~\ref{fig:conv} and \ref{fig:recepoch_a} also imply that existing PS and hybrid communication techniques cannot fundamentally solve the communication bottleneck in large-scale embedding model training. Fortunately, due to the fine-grained caching and consistency, our proposed \name achieves significant performance improvement and up to 88\% ($\approx 1-1/8.2$) embedding communication reduction. Figure~\ref{fig:recepoch_b} shows the per epoch time and communication time speedup comparison on three DLRM tasks under 10 Gbit Ethernet cluster. As shown in Figure~\ref{fig:recepoch_b}, although the speedups are smaller than those in the 1 Gbit Ethernet cluster due to the higher network bandwidth, the communication costs are still the bottleneck of the model training process. We see that our system still outperforms these baselines and achieves up to 2.3$\times$ and 5.2$\times$ speedup compared to HET Hybrid and TF Parallax respectively. Another interesting finding is that HET AR on 1 Gbit Ethernet cluster performs better than HET PS due to the utilization of the PCIe bandwidth cross GPUs. Although AllReduce degenerates to the inefficient AllGather primitive for sparse communication, it still achieves similar performance as HET Hybrid because these baselines involving the PS are suffering from the limited network bandwidth between servers and workers. When we use a 10 Gbit Ethernet, the high network bandwidth significantly improves the speed of PS-based methods. But HET AR maintains similar performance and becomes the slowest among all methods.

\begin{table}[t]
\caption{Final test AUC (\%) with different $s$ on Criteo}
\label{tab:accuracy}
\vspace{-2mm}
\centering
{
\noindent
\renewcommand{\multirowsetup}{\centering}
\resizebox{1.0\linewidth}{!}{
\begin{tabular}{ccccc|ccc}
\toprule
\textbf{Models}&textbf{$s=0$}&\textbf{$s=100$}&\textbf{$s=10k$}&\textbf{$s=\infty$} & \multirow{2}{*}{ \begin{tabular}[c]{@{}c@{}} \textbf{Cache miss} \\ \textbf{rate (WDL)}\end{tabular}} & \multirow{2}{*}{\begin{tabular}[c]{@{}c@{}} $s=0$ \end{tabular}} & \multirow{2}{*}{\begin{tabular}[c]{@{}c@{}} $s=100$ \end{tabular}} \\
\cmidrule(l){1-5} 
WDL & 79.64 & 79.62 & 78.84 & 74.61 \\
\cmidrule(l){6-8} DFM & 79.74 & 79.73 & 78.96 & 70.91 & 0\% & 80.17\% & 80.15\% \\
DCN & 80.25 & 80.24 & 79.76 & 75.29 & >0\% & 78.17\% & 78.12\%
\\
\bottomrule
\end{tabular}}
\vspace{-2mm}
}
\end{table}

\begin{figure}[t]
\centering
\includegraphics[width=0.95\linewidth]{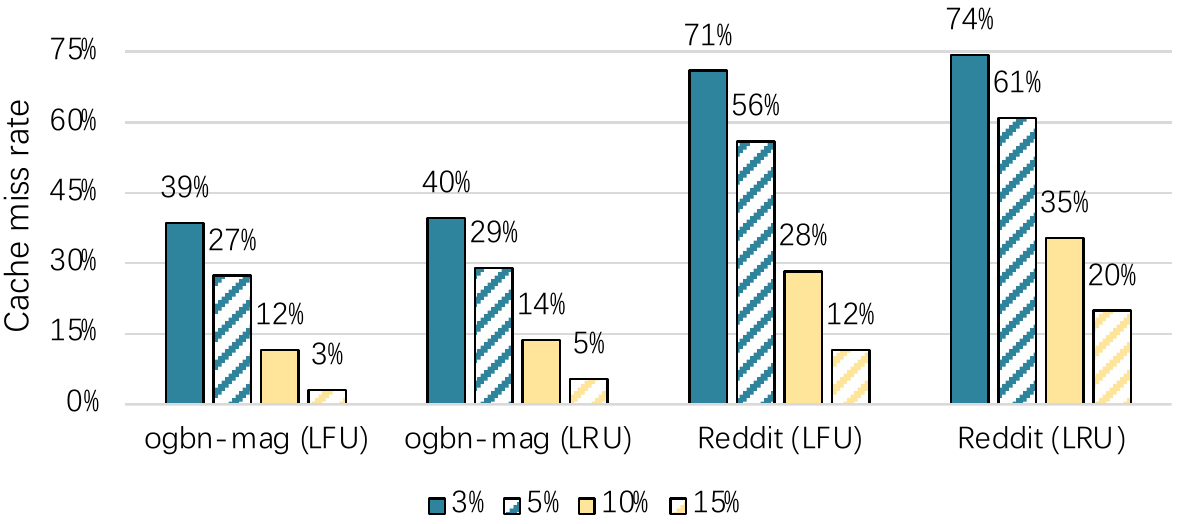}
\vspace{-2mm}
\caption{Cache miss rate under different cache space and strategy settings on GNN tasks.}
\vspace{-2mm}
\label{fig:cache_miss}
\end{figure}

\textit{\textbf{Convergence quality.}}  
We first investigate how much negative impact could the staleness have on the accuracy. We illustrate the convergent model performance (i.e., test AUC) with different staleness thresholds in Table~\ref{tab:accuracy} (left part). 
Due to the inherent robustness of iterative convergent algorithms, we find that our method reaches the target model quality even under moderate
levels of staleness ($s=100$), although the model degradation becomes apparent under high
staleness levels.

We next investigate pathological cases (e.g. is it not possible that prediction with the embedding parameters that are less frequently synchronized may result in less accurate results (i.e., bias)?) We make a further study on the test dataset of Criteo on WDL. As a cache hit (miss) implies the prediction uses stale (up-to-date) embedding parameters, we use cache miss rate to measure the frequency of the prediction using the stale (less frequently synchronized) embedding parameters.
Therefore, we split the test set into two sets based on the cache miss rate. As shown in Table~\ref{tab:accuracy} (right part), the predictions distribution from two models ($s=0$ and $s=100$) are very close, which demonstrates that the stale embeddings will not incur significant predication bias.

\subsection{System Configuration Sensitivity}
We study the impact of different cache embedding table sizes and cache strategy settings. We measure the cache miss rate on GNN task with ogbn-mag and Reddit datasets on cluster A. As shown in Figure~\ref{fig:cache_miss}, LFU often performs a lower cache miss rate than LRU. This is because LFU could reflect the long-term embedding access popularity better. We also evaluate different cache embedding table sizes, including 3\%, 5\%, 10\% and 15\%. As the cache table size growing, the cache miss rate significantly decreases. For ogbn-mag with LFU, given a piece of cache space whose size equals 15\% of the global embedding table size, almost 97\% embedding accesses are performed on the cache embedding table. This experiment strongly verifies the effectiveness of \name and explains how does our cache embedding table help to reduce the communication costs.

\begin{figure}[t]
	\subfigure[WDL-Criteo speedup]{
	    \includegraphics[width=0.47\linewidth]{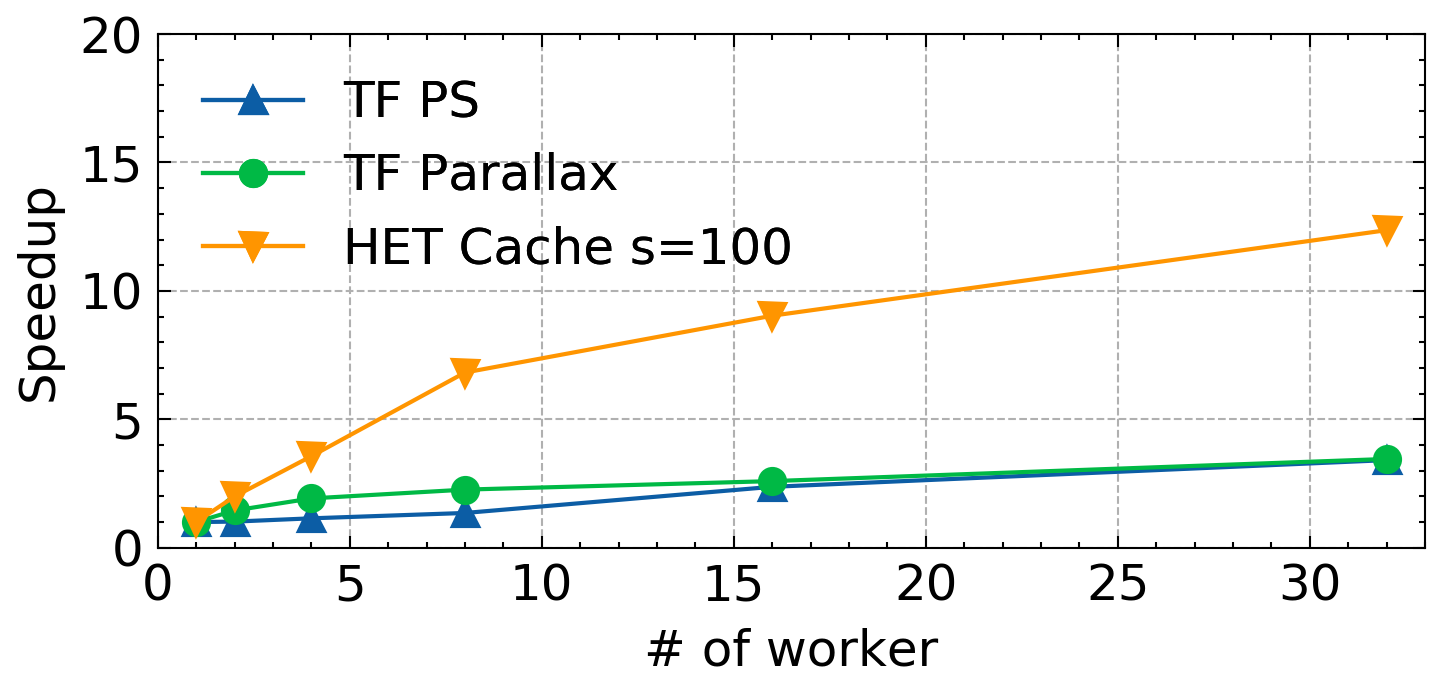}
	    \label{fig:speedwdl}
	}
	\subfigure[GNN-Reddit speedup]{
        \includegraphics[width=0.47\linewidth]{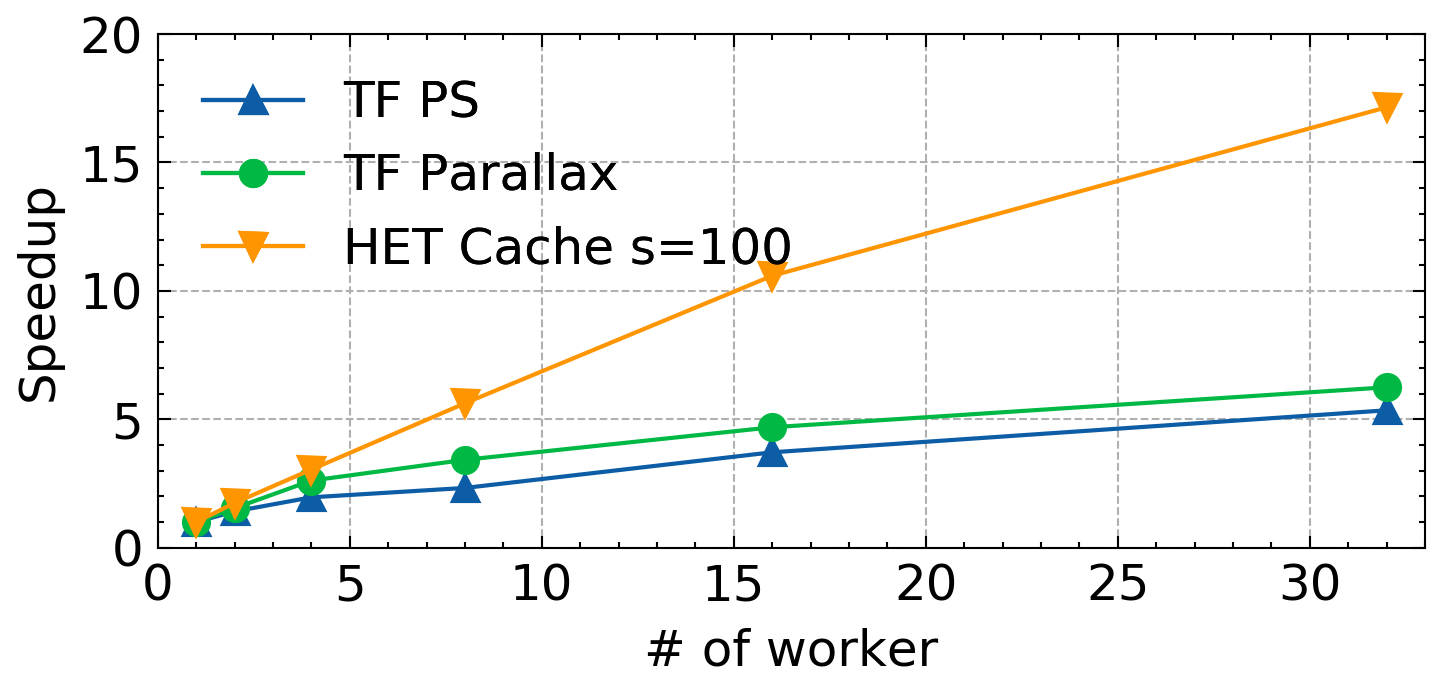}
        \label{fig:speedreddit}
    }
    \vspace{-4mm}
    \subfigure[WDL-Criteo with different embedding hidden size $D$]{
        \vspace{-2mm}
        \includegraphics[width=0.95\linewidth]{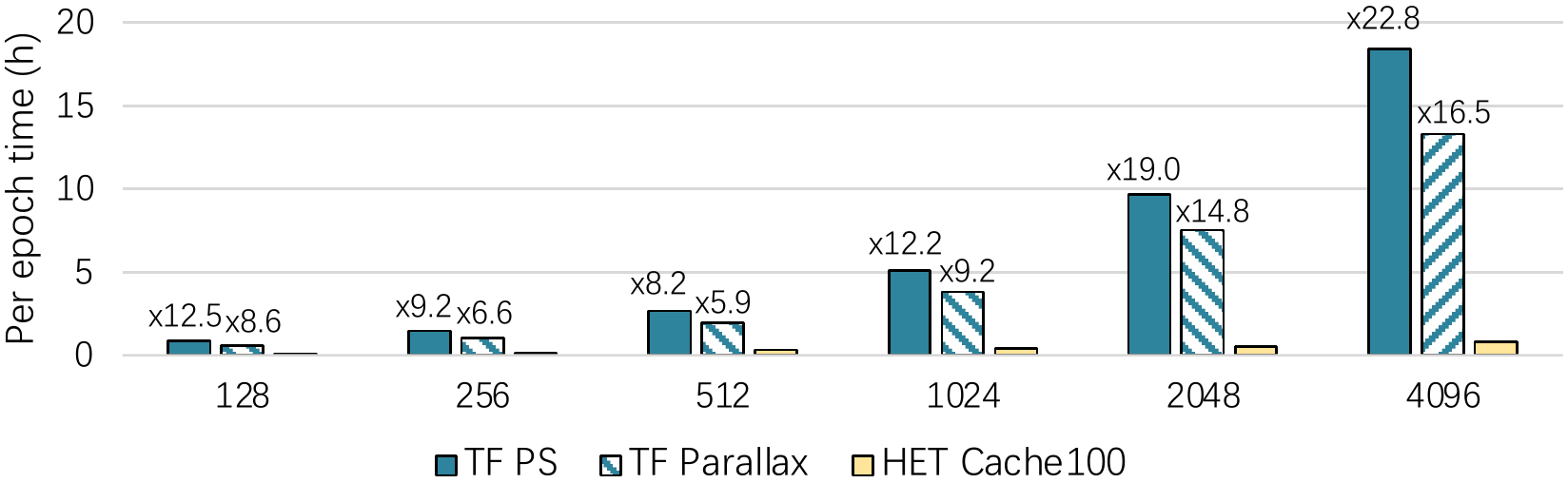}
        \label{fig:modelwdl}
    }
    \caption{Scalability study}\label{fig:scala}
    \vspace{-4mm}
\end{figure}

\subsection{Scalability}
\label{sec:scala}
We conduct a scalability study in terms of run time speedup on cluster A over 4 servers with 1, 2, 4, 8, 16, and 32 workers respectively. As shown in Figure~\ref{fig:speedwdl} and Figure~\ref{fig:speedreddit}, both TF PS and TF Parallax have limited scalability suffering from large amounts of embedding communication costs. By contrast, our \name achieves improved scalability by caching the hot embeddings. We also note that all methods show better scalability on GNN-Reddit than WDL-Criteo. Because the latter has a larger embedding table and becomes more communication-intensive and more difficult to scale up. 
We also study the model scalability of \name on WDL-Criteo by increasing the embedding size $D$ up to 4096 (around \textit{one trillion model parameters}) over 32 workers. Figure~\ref{fig:modelwdl} shows that \name significantly outperforms TF and Parallax since their PS architecture faces a more serious communication bottleneck with such a large scale embedding table.

\section{Conclusion}
Embedding model trained on high-dimensional data is common at modern web companies and poses
an extra challenge to standard frameworks: the high communication overhead causes the embedding workloads to have low
execution efficiency and scalability. To address this performance bottleneck, we presented \name, a system framework leveraging the
embedding cache architecture combined
with fine-grained consistency and stale-write protocols. 
Experimental results have shown that \name could reduce
up to 88\% embedding communication and achieve up to $20.68\times$ performance
improvements, compared to the state-of-the-art baselines. We hope that this work and the
open-source release of \name helps motivate the release
of larger high-dimension datasets from modern web companies and the increase of research on larger embedding models.

\section*{Acknowledgments}
This work is supported by the National Key Research and Development Program of China (No. 2018YFB1004403), the National Natural Science Foundation of China under Grant (No. 61832001), Beijing Academy of Artificial Intelligence (BAAI), PKU-Baidu Fund 2019BD006, and PKU-Tencent Joint Research Lab. Bin Cui and Zhi Yang are the co-corresponding authors.
  
\newpage
\balance
\bibliographystyle{ACM-Reference-Format}
\bibliography{my-reference}


\end{document}